%% file: example.tex
\documentclass{article}

\usepackage[final]{corl_2026} 

\usepackage{microtype}
\usepackage{graphicx}

\usepackage{booktabs} 
\usepackage{enumitem} 
\usepackage{subcaption}
\usepackage{wrapfig}

\usepackage{amsmath}
\usepackage{amssymb}
\usepackage{mathtools}
\usepackage{amsthm}
\theoremstyle{plain}
\newtheorem{theorem}{Theorem}[section]
\newtheorem{proposition}[theorem]{Proposition}

\newtheorem{corollary}[theorem]{Corollary}
\theoremstyle{definition}

\theoremstyle{remark}

\usepackage{algorithm}
\usepackage{algpseudocode}

\title{Diagnosing Compositional Generalization in Sequential Robot Tasks}

\author{
\normalfont
Yixiao Wang,
Cheng-En Wu,
Lingfeng Sun,
Pengcheng Wang,
Xiang Ji,\\
Boyuan Liang,
Guojian Zhan,
Masayoshi Tomizuka\\[4pt]
\small
University of California, Berkeley
}

\begin{document}
\maketitle


\input{sections/abs}


\input{sections/intro}

\input{sections/related}

\input{sections/prob}

\input{sections/theorem}

\input{sections/exp_v1}

\section{Conclusion and Limitations}
\label{sec:conclusion}

We presented a theory-driven study of compositional generalization in sequential robot tasks. Our analysis decomposes the generalization gap into \textit{marginal instruction shift}, \textit{instruction-compositional shift}, and \textit{context--action shift}. Guided by this decomposition, our experiments provide three main insights. First, exhaustive enumeration of all task tuples is unnecessary; a carefully structured subset, even one quarter of the full task space, can be sufficient when it covers the dependencies that matter for action prediction. Second, sparse training often fails not because the policy lacks subtask skills, but because it cannot correctly steer those skills under unseen instruction recombinations: finetuning with only one demonstration per task can improve OOD success from \(0.4\%\) to \(54.7\%\). Third, when task factors are semantically dependent, effective coverage must capture the relational structure that determines the correct action, rather than only ensuring factor diversity. Overall, these findings suggest that efficient robot data collection should prioritize dependency coverage in the instruction space instead of exhaustively expanding the task set.

This study has several limitations. First, while our experiments cover both independent and dependent instructions, they are restricted to a small set of sequential manipulation tasks, as the required demonstrations and evaluation tasks grow exponentially with instruction complexity; extending this analysis to large-scale training and characterizing how such coverage patterns inform pretraining are important directions for future work. Second, our analysis centers on policies trained from scratch atop a frozen vision encoder; a natural extension is to pretrained vision-language-action (VLA) models, both to diagnose the origins of compositional generalization and to assess how our coverage principle transfers to this setting.


\clearpage
\acknowledgments{This material is based upon work supported by FANUC.}


\bibliography{example}  

\input{sections/app}

\end{document}

%% file: sections/abs.tex
\begin{abstract}
Sequential robot manipulation requires policies to execute novel combinations of familiar instruction components. However, collecting demonstrations for all possible instruction tuples is combinatorially expensive, while sparsely covered datasets often fail under out-of-distribution recombination. This paper studies compositional generalization through the lens of instruction-space coverage. We decompose the generalization gap into three sources: \textit{marginal instruction shift}, \textit{instruction-compositional shift}, and \textit{context--action shift}. This decomposition allows us to diagnose when sparse training coverage is sufficient, and what structure the training set must preserve for reliable action prediction. Our results show that exhaustive tuple enumeration is unnecessary: a structured subset can recover strong out-of-distribution performance when it covers action-relevant dependencies. We further find that sparse training often fails due to instruction steering rather than missing low-level skills; finetuning with only one demonstration per task improves OOD success from \(0.4\%\) to \(54.7\%\). For semantically dependent tasks, effective coverage must capture relational structure rather than only factor diversity. These findings suggest that efficient robot data collection should prioritize dependency coverage in instruction space over exhaustive task expansion. Additional results, code, and dataset are available on the \href{https://yixiaowang7.github.io/Diagnosing_Compositional_Generalization_Robot_Page/}{project website}.
\end{abstract}

\keywords{Compositional Generalization; Data-Efficient Imitation Learning} 

%% file: sections/intro.tex
\section{Introduction}

Sequential robot manipulation often requires executing a \emph{composite} instruction composed of multiple subtasks. A robot may need to pick a specified object, place it into a specified container, and then interact with another target. Existing approaches typically handle such problems in two ways. One family of methods uses a hierarchical pipeline, where a high-level controller selects the active subtask and a low-level policy executes the corresponding atomic instruction~\cite{ahn2022icanisay, pi0.5, geminirobotics}. Another family trains end-to-end instruction-conditioned policies that map observations and the full instruction directly to actions~\cite{team2024octo, fan2025interleave}. Both paradigms can perform well when training and test instructions substantially overlap. However, real deployments require \emph{systematic recombination}: the robot must execute new instruction tuples whose individual components are familiar, but whose combinations were rarely or never observed during training.

A naive response is to collect demonstrations for every possible instruction tuple. This strategy is usually impractical, because the number of task combinations grows combinatorially with the number of instruction factors. In practice, robot datasets are often \emph{compositionally sparse}: they cover individual objects, targets, and primitive behaviors, but only a small fraction of their possible recombinations. This raises a fundamental question for robot data collection: \emph{what coverage of the instruction space is actually necessary for compositional generalization?}

This paper studies compositional generalization from the perspective of \emph{instruction-space coverage}. We formalize the generalization gap between a training distribution and a test distribution over held-out recombinations of familiar instruction factors. The resulting analysis decomposes the gap into three sources: \textit{marginal instruction shift}, \textit{instruction-compositional shift}, and \textit{context--action shift}. This decomposition highlights a key distinction: observing every atomic instruction value is not enough if the training data do not cover the dependencies among instruction factors that determine the action distribution. Conversely, full Cartesian coverage is not always necessary when the relevant dependency structure is sufficiently covered. This perspective also clarifies the role of stage-wise modularity. A modular policy that conditions only on the active subtask can reduce unnecessary dependence on irrelevant instruction factors, leading to a sharper stage-wise generalization bound. However, modularity alone is not a complete solution. Inactive instruction factors can influence the contexts, and modularity requires a high-level planner that decomposes the task and generates dependency-preserving subtask instructions. Thus, the central issue is not only how the policy is parameterized, but also whether the training support identifies the instruction dependencies that govern behavior.

Guided by this analysis, we conduct controlled sequential robot manipulation experiments with both independent and semantically dependent instruction structures. The experiments provide three main insights. First, exhaustive enumeration of all task tuples is unnecessary; a carefully structured subset, even one quarter of the full task space, can be sufficient when it covers the dependencies that matter for action prediction. Second, sparse training often fails not because the policy lacks reusable subtask skills, but because it cannot correctly steer those skills under unseen instruction recombinations. Finetuning with only one demonstration per task improves OOD success from \(0.4\%\) to \(54.7\%\), supporting this interpretation. Third, when task factors are semantically dependent, effective coverage must capture the relational structure that determines the correct action, rather than only ensuring factor diversity. Together, these results suggest a practical principle for robot data collection: demonstrations should be allocated to cover action-relevant dependency structure instead of exhaustively expanding the task set.

%% file: sections/related.tex
\section{Related Work}

\paragraph{Language-Conditioned Robot Policies.}
Recent robot learning systems train generalist policies that map observations and language or multimodal prompts to actions~\cite{zitkovich2023rt, team2024octo, kim2024openvla, pi0.5, fan2025interleave}. These policies typically use Transformer backbones~\cite{vaswani2017attention}, with action heads ranging from autoregressive token prediction~\cite{pertsch2025fast} to diffusion- or flow-based generation~\cite{diffusionpolicy, pi0.5}. Multimodal systems such as VIMA and Interleave-VLA further show how multimodal tokens can specify diverse manipulation tasks under a unified interface~\cite{jiang2023vimageneralrobotmanipulation, fan2025interleave}. While these works demonstrate scalable instruction-conditioned imitation learning, our question is more specific: when the instruction space is combinatorially large, which instruction tuples must be covered to generalize to held-out recombinations?

\paragraph{Compositional Generalization in Language and Robotics.}
Systematic generalization has long been studied in language and sequence modeling, where models can perform well on random splits but fail on recombinations of familiar primitives~\cite{lake2018generalization, yagcioglu2024sequential}. Robotics adds difficulty because the instruction tuple also changes the induced context and action distribution. Benchmarks such as VIMA-Bench, ClevrSkills, and RoboHiMan evaluate held-out combinations of objects, attributes, skills, or long-horizon task structures~\cite{jiang2023vimageneralrobotmanipulation, haresh2024clevrskills, chen2025robohiman}. Our work studies the data-side mechanism behind these failures: we decompose the generalization gap into \textit{marginal instruction shift}, \textit{instruction-compositional shift}, and \textit{context--action shift}, and test how training support controls these terms.

\paragraph{Data Coverage and Combinatorial Designs.}
Recent work reduces robot data cost by scaling demonstrations, generating trajectories, or varying environmental factors. MimicGen increases data coverage by object-centric subtask motion transformation~\cite{mandlekar2023mimicgen}. Closest to our motivation, Gao et al.~\cite{gao2024efficient} study efficient data collection by heuristically varying environmental factors. We instead study sequential instruction factors and ask how to design training coverage. This connects to design-of-experiments and combinatorial testing, where orthogonal and covering arrays cover low-order factor interactions without enumerating the full Cartesian product~\cite{kacker1991taguchi, hedayat2012orthogonal, kuhn2013introduction}.

\paragraph{Modularity and Hierarchy.}
Modularity and hierarchy are long-standing inductive biases for long-horizon control, spanning options and hierarchical reinforcement learning~\cite{bacon2017option, nachum2018data}, task-and-motion planning~\cite{garrett2021integrated}, and language-guided planner-controller systems~\cite{ahn2022icanisay, pi0.5, hirobot, geminirobotics}. These systems can reduce coupling between inactive instructions and current actions by exposing an active stage or subgoal. Our analysis formalizes this benefit through a stage-wise modular policy, which removes the explicit instruction-compositional term from the gap. However, modularity alone is insufficient: inactive instructions can still determine future contexts, and dependent tasks require the planner or dataset to communicate cross-stage relations. Our contribution is therefore a data coverage principle for deciding when sparse instruction data is enough and when it leaves systematic relational errors.

%% file: sections/prob.tex
\section{Problem Statement}
We formalize sequential robot tasks as the execution of an ordered $n$-tuple of discrete subtask indices $l = (l_1, \dots, l_n) \in \mathcal{L}_1 \times \dots \times \mathcal{L}_n:=\mathcal{L}$, where each $\mathcal{L}_j = \{0, \dots, m\}$. Let $z \in \mathcal{Z}$ denote the time-$t$ context (e.g., observation and robot state) and $a \in \mathcal{A}$ be the robot action. A robot policy is modeled as a conditional action distribution, $p(a \mid z, l_1, \dots, l_n)$. The execution is naturally decomposed into stages governed by a stage indicator $\sigma: \mathcal{Z} \to \{1, \dots, n\}$, such that $\sigma(z) = i$ signifies that subtask $l_i$ is currently active.

\subsection{Compositional Generalization Gap}
We assume the training data are sampled from a joint distribution $p(z,a,l)$, where $l=(l_1,\dots,l_n)\in \mathcal{L}_1\times\cdots\times\mathcal{L}_n$ denotes a compositional instruction. The training support is
\begin{equation}
\label{eq:train-set}
\begin{aligned}
    E_{\mathrm{train}} :&= \operatorname{supp}(p(l)) \\
    &= \left\{ l \in \mathcal{L}_1 \times \dots \times \mathcal{L}_n \;\big|\; p(l) > 0 \right\}.
\end{aligned}
\end{equation}
We introduce a test-time joint distribution \(q(z,a,l)\), whose instruction marginal \(q(l)\) has support \(E_{\mathrm{total}}:=\operatorname{supp}(q(l))\) over valid recombinations of familiar factors. The OOD set is then
\begin{equation}
E_{\mathrm{ood}} := E_{\mathrm{total}} \setminus E_{\mathrm{train}}.
\end{equation}

We define the expected risk $\mathcal{R}$ over a distribution $\mathcal{D} \in \{p, q\}$ as:
\begin{equation}
    \mathcal{R}_{\mathcal{D}}(\theta) := \int_{\mathcal Z \times \mathcal A \times \mathcal L} \mathcal{D}(z,a,l) L_\theta(z,a,l)\,dz\,da\,dl.
\end{equation}
where \(l_{-i}\) denotes all instruction components except \(l_i\), \(\theta\) denotes the policy parameters, and \(L_\theta\) denotes the policy loss, instantiated as flow-matching loss in our paper. 

The \textit{compositional generalization gap} quantifies the performance degradation caused by the distribution shift from training to testing:
\begin{equation}
\label{eq: compostional generalization gap}
\begin{aligned}
    \Delta_q(\theta) :=& \mathcal{R}_{q}(\theta) - \mathcal{R}_{p}(\theta) 
    =\int_{\mathcal Z \times \mathcal A \times \mathcal L} \bigl(q(z,a,l)-p(z,a,l)\bigr)\, L_\theta(z,a,l)\,dz\,da\,dl 
\end{aligned}
\end{equation}

%% file: sections/theorem.tex
\section{Theoretical Analysis}
\label{sec: theoretical analysis}
We begin with Proposition~\ref{prop:general_gap_bound} (proof is provided in Appendix~\ref{app: theoretical analysis}), which identifies three distinct sources of the compositional generalization gap. The decomposition is centered on the instruction sequence \(l\), since the gap arises from a mismatch between the training and test instruction distributions.

\begin{proposition}[General upper bound on the compositional generalization gap]
\label{prop:general_gap_bound}
Assume the loss is measurable and uniformly bounded on
\(\operatorname{supp}(p)\cup \operatorname{supp}(q)\), i.e.,
\[
0 \le L_\theta(z,a,l) \le M
\qquad
\forall (z,a,l)\in \operatorname{supp}(p)\cup \operatorname{supp}(q).
\]
Then, for any fixed \(i \in \{1,\dots,n\}\),
\begin{equation}
\label{eq:general_gap_bound}
\begin{aligned}
|\Delta_q(\theta)|
\le M \Big(
&\|q(l_i)-p(l_i)\|_1 + \mathbb{E}_{l_i \sim p(l_i)}
\big[\|q(l_{-i}\mid l_i)-p(l_{-i}\mid l_i)\|_1\big] \\
&+ \mathbb{E}_{l \sim p(l)}
\big[\|q(z,a\mid l)-p(z,a\mid l)\|_1\big]
\Big),
\end{aligned}
\end{equation}
where 
\[
\|\mu-\nu\|_1 := \int_{\mathcal X} |\mu(x)-\nu(x)|\,dx
\]
for two distributions \(\mu,\nu\).
\end{proposition}

From Proposition~\ref{prop:general_gap_bound}, bounded training loss alone does not imply good compositional generalization. Rather, a small gap requires simultaneous control of three discrepancies. The first term, \(\|q(l_i)-p(l_i)\|_1\), measures \textit{marginal shift} in the subtask instruction \(l_i\), which is unavoidable without broader training coverage. The second term, \(\mathbb{E}_{l_i\sim p(l_i)}[\|q(l_{-i}\mid l_i)-p(l_{-i}\mid l_i)\|_1]\), captures \textit{instruction-compositional shift}, namely whether familiar subtasks are recombined in unfamiliar ways. The third term, \(\mathbb{E}_{l\sim p(l)}[\|q(z,a\mid l)-p(z,a\mid l)\|_1]\), measures \textit{context--action shift} conditioned on the full instruction sequence caused by the previous term. The latter two terms are amplified when the policy entangles the active subtask with irrelevant inactive instructions. Hence, once marginal coverage is ensured, compositional generalization depends primarily on how the policy uses instruction structure. This motivates a policy that can isolate stage-relevant instruction information while suppressing unnecessary dependence on \(l_{-i}\).

\subsection{Stage-wise Modular Policy}
\label{sec: stage wise modular policy}
The preceding analysis shows that compositional generalization depends critically on instruction structure. This motivates the hierarchical parameterization adopted in systems such as SayCan~\cite{ahn2022icanisay} and Gemini Robotics~\cite{geminirobotics}, which we term a \emph{stage-wise modular policy}. Under this formulation, a high-level oracle identifies the active stage, and a low-level controller conditions only on the current context and the corresponding subtask instruction. Formally, let \(\sigma:\mathcal Z\to\{1,\dots,n\}\) denote the stage indicator. Then, at context \(z\),
\[
\pi_\theta(a\mid z,l)=\pi_{\theta,\sigma(z)}(a\mid z,l_{\sigma(z)}).
\]
This modularization isolates the active subtask and leads to a sharper stage-wise generalization bound.

Under this modular parameterization, the loss can be written in stage-wise form as
\[
L_\theta(z,a,l)
=
\sum_{i=1}^n \mathbf 1\{\sigma(z)=i\}\,\ell_{i,\theta}(z,a,l_i),
\]
where \(\ell_{i,\theta}(z,a,l_i)\) denotes the loss incurred at stage \(i\) and depends only on the active subtask \(l_i\).

\begin{corollary}[Bound for stage-wise modular policy]
\label{cor:modular_gap_bound}
For a stage-wise modular policy, we assume that for each stage \(i\),
\[
0 \le \ell_{i,\theta}(z,a,l_i) \le M_i
\quad
\forall (z,a,l_i)\in \operatorname{supp}(p_i)\cup \operatorname{supp}(q_i),
\]
where $q_i=\int q(z,a,l)dl_{-i}$ and $p_i=\int p(z,a,l)dl_{-i}$. Then
\begin{equation}
\label{eq:modular_gap_bound}
\begin{aligned}
|\Delta_q^{\mathrm{m}}(\theta)|
\le
\sum_{i=1}^n M_i \Big(
&\|q(l_i)-p(l_i)\|_1 +
\mathbb{E}_{l_i\sim p}\|q(z,a\mid l_i)-p(z,a\mid l_i)\|_1\Big).
\end{aligned}
\end{equation}
\end{corollary}

Corollary~\ref{cor:modular_gap_bound} (proof is provided in Appendix~\ref{app: theoretical analysis}) shows that stage-wise modularity removes the explicit \textit{instruction-compositional shift} and yields a simpler stage-wise training loss. This simplification depends critically on access to such an indicator, or on learning it accurately. More fundamentally, removing \(l_{-i}\) from the policy input does not remove its influence on the induced context--action distribution \(p(z,a\mid l_i)\). Unseen values of \(l_{-i}\) can still generate novel contexts, under which the policy still fails to generalize. Moreover, such modularization may induce an information bottleneck: if the optimal action depends on instruction factors outside \(l_i\), then \(p(a\mid z,l_i)\) is inherently ambiguous.

%% file: sections/exp_v1.tex
\section{Experiments}
\label{sec:experiments}
Our experiments study the following question: \emph{what coverage of the instruction space is necessary for compositional generalization in sequential robot tasks?} Collecting demonstrations for every instruction tuple is sufficient but scales poorly, as the number of tuples grows exponentially with the number of subtasks. Proposition~\ref{prop:general_gap_bound} suggests a theoretical view: since the generalization gap is bounded by marginal, instruction-compositional, and context--action shifts, controlling these terms should matter more than exhaustive enumeration. Guided by this decomposition, we study three questions: \emph{Q1: Is full tuple coverage necessary?} \emph{Q2: What is missing when sparse coverage fails?} \emph{Q3: What coverage is needed for dependent instructions?}

\begin{figure}[t]
    \centering
    \includegraphics[width=0.95\linewidth]{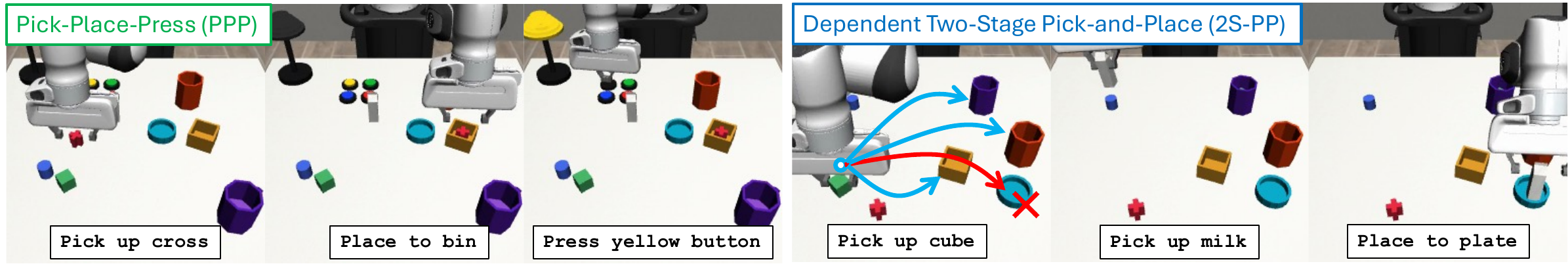}
    \caption{
    Illustration of Pick-Place-Press and Dependent Two-Stage Pick-and-Place tasks.
    }
    \label{fig:tasks}
\end{figure}

\begin{figure}[t]
    \centering
    \begin{subfigure}[t]{0.48\linewidth}
        \centering
        \includegraphics[width=\linewidth]{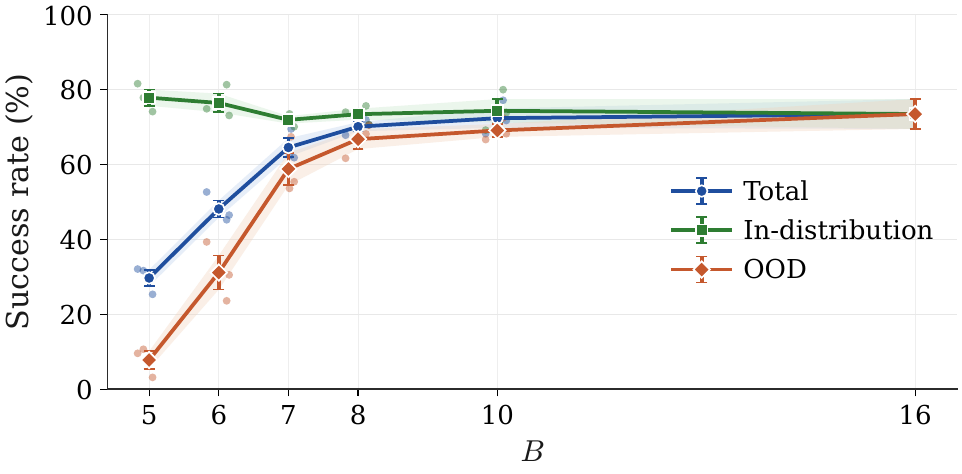}
        \caption{Pick-and-Place (PP).}
        \label{fig:pp_budget}
    \end{subfigure}
    \hfill
    \begin{subfigure}[t]{0.48\linewidth}
        \centering
        \includegraphics[width=\linewidth]{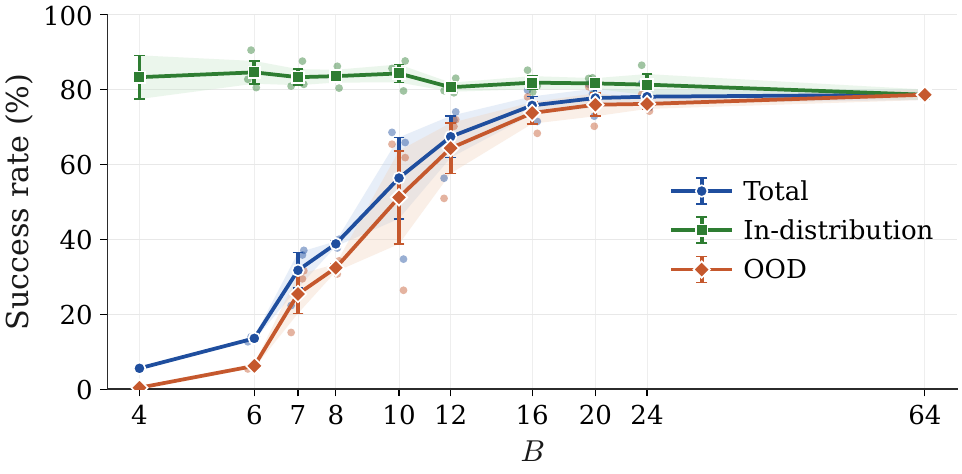}
        \caption{Pick-Place-Press (PPP).}
        \label{fig:ppp_budget}
    \end{subfigure}
    \caption{Success rates versus training-set size \(B\) under a fixed total demonstration budget.}
    \label{fig:coverage_ood}
    \vspace{-12pt}
\end{figure}


\subsection{Experimental Setup}
\label{sec:exp_setup}

\paragraph{Task Design.}
We design three tasks with increasing compositional complexity in robomimic~\citep{robomimic2021}. Figure~\ref{fig:tasks} illustrates PPP and 2S-PP.

\begin{enumerate}
    \item \textbf{Pick-and-Place (PP).}
    The instruction is \(l=(o,c)\), where \(o\in\{1,\dots,4\}\) is the object and \(c\in\{1,\dots,4\}\) is the container. The robot picks object \(o\) and places it into container \(c\). The full task space contains \(|E_{\mathrm{total}}|=4\times4=16\) tasks.

    \item \textbf{Pick-Place-Press (PPP).} The instruction is \(l=(o,c,b)\), where \(b\in\{1,\dots,4\}\) denotes the button color. The robot places object \(o\) into container \(c\), then presses button \(b\), activating the lamp with the corresponding color. The full task space contains \(|E_{\mathrm{total}}|=4\times4\times4=64\) tasks. The three subtasks are independent. 
    
    \item \textbf{Dependent Two-Stage Pick-and-Place (2S-PP).} The full task is \(l=(o_1,c_1,o_2,c_2)\), with constraints \(o_1\neq o_2\) and \(c_1\neq c_2\). The robot first places \(o_1\) into \(c_1\), then places \(o_2\) into \(c_2\). The full task space contains \(|E_{\mathrm{total}}|=4\times4\times3\times3=144\) tasks. \textbf{However, the policy only receives the partial instruction \(l=(o_1,o_2,c_2)\).} Thus, it must infer the missing container \(c_1\) for \(o_1\) while satisfying \(c_1\neq c_2\), making the task dependent and more challenging. For each training instruction \((o_1,o_2,c_2)\), demonstrations cover all valid \(c_1\neq c_2\), although \(c_1\) is omitted from the policy input.
\end{enumerate}

\paragraph{Data and Policy.} For these three benchmarks, we generate demonstrations with a customized motion planner and train a transformer-based flow-matching policy with a frozen DINOv2~\cite{oquab2024dinov2} encoder. The policy predicts action \(a\) conditioned on instruction \(l\) and the current context \(z\), comprising visual observations and robot state. For comparisons across training-set sizes \(B\), we hold the total demonstration budget fixed and sample uniformly over the \(B\) training tasks. Random seeds control demonstration sampling and network initialization. Appendix~\ref{app:data_training}--\ref{app:network} provides more details.

\paragraph{Evaluation Protocol.} We report success rates over all tasks \(E_{\mathrm{total}}\) (\textbf{Total SR}), training tasks \(E_{\mathrm{train}}\) (\textbf{ID SR}), and held-out recombinations \(E_{\mathrm{ood}}=E_{\mathrm{total}}\setminus E_{\mathrm{train}}\) (\textbf{OOD SR}). For full-task training, no held-out tuples remain; the overall success rate is shown as a reference in both ID and OOD success plots. For 2S-PP, we additionally report the \textbf{same-container violation rate}: the probability of placing \(o_1\) into \(c_2\), conditioned on completing the first placement. Unless otherwise noted, we evaluate three training seeds with 30 trials per task per seed. More details are provided in Appendix~\ref{app:evaluation}.

\subsection{Q1: Is Full Tuple Coverage Necessary?}
\label{sec:exp_full_combinations}

We first test whether full Cartesian coverage is required for compositional generalization. For each training budget \(B=|E_{\mathrm{train}}|\), we sample training tasks while ensuring that every subtask instruction value appears at least once. This controls the \textit{marginal shift term} \(\|q(l_i)-p(l_i)\|_1\) in Proposition~\ref{prop:general_gap_bound}. Therefore, poor OOD SR cannot be attributed merely to unseen subtask instructions.

Figure~\ref{fig:coverage_ood} shows that small training sets still yield poor OOD success despite marginal coverage. However, full tuple coverage is nevertheless unnecessary: PP saturates around \(B=8\), while PPP saturates around \(B=16\), well below the sizes of the full \(16\)- and \(64\)-task Cartesian products, respectively, while achieving similar success. Strong OOD success suggests that, when \(B\) is large enough, low loss \(L_\theta(z,a,l)\) may generalize to unseen instruction tuples and their induced contexts, and the \textit{context-action shift} appears insignificant. To examine this directly, we evaluate flow-matching loss on 1,000 demonstrations from unseen instructions. For checkpoints trained with \(B=4,8,16\), ID loss remains nearly constant at \(0.00648\), \(0.00663\), and \(0.00667\), respectively. In contrast, OOD loss decreases from \(0.03421\) to \(0.01548\) and \(0.00756\), while OOD success increases from \(0.83\%\) to \(23.93\%\) and \(82.29\%\). The narrowing ID--OOD loss gap supports the view that sufficient training diversity enables the policy to achieve loss on unseen instruction tuples that approaches its loss on seen tuples. This pattern may reflect a learned focus on subtask-relevant information.

To identify the relevant structure, we analyze the PPP task at the level of instruction pairs. For each held-out task \(l=(o,c,b)\), we define its \emph{seen pair count} as the number of observed pairs among \((o,c)\), \((o,b)\), and \((c,b)\). Figure~\ref{fig:pair_alignment_bars} shows that OOD success generally increases with seen pair count at a fixed training budget \(B\), especially before saturation. Thus, pairwise coverage helps explain generalization beyond training-set size alone. Figure~\ref{fig:pair_alignment_cubes} illustrates representative examples where OOD task success correlates strongly with the number of seen pairs. These results also support Proposition~\ref{prop:general_gap_bound}: pairwise coverage mitigates the instruction-compositional discrepancy $q(l_{-i}\mid l_i)-p(l_{-i}\mid l_i)$ and consequently reduces the induced context--action shift.

\begin{figure}[!t]
    \centering
    \begin{subfigure}[t]{0.49\linewidth}
        \centering
        \includegraphics[width=\linewidth]{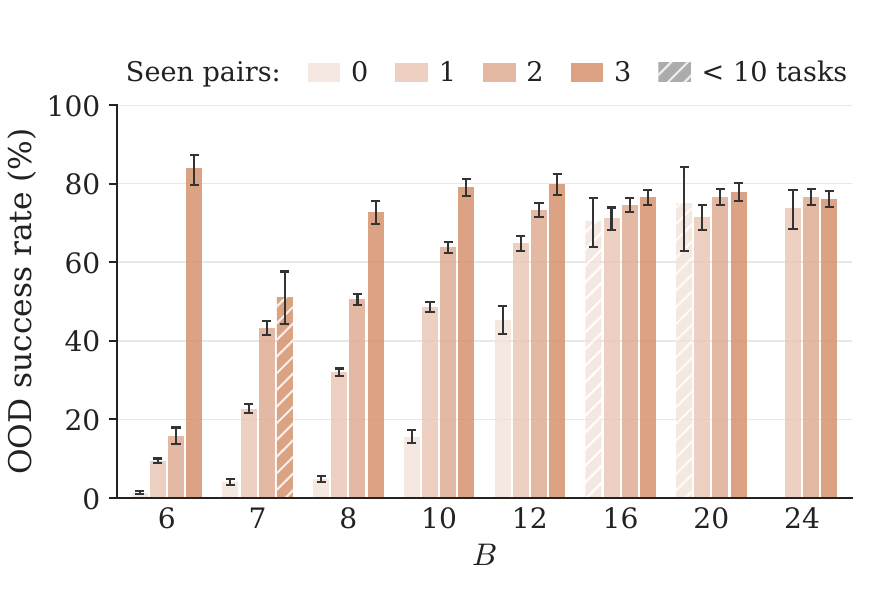}
        \caption{OOD success in PPP grouped by seen pair count, for each training budget $B$. Hatched bars contain fewer than 10 held-out tasks. Within each $B$, OOD success increases with the seen pair count, and the effect is largest before saturation ($B\le 12$).}
        \label{fig:pair_alignment_bars}
    \end{subfigure}
    \hfill
    \begin{subfigure}[t]{0.49\linewidth}
        \centering
        \includegraphics[width=\linewidth]{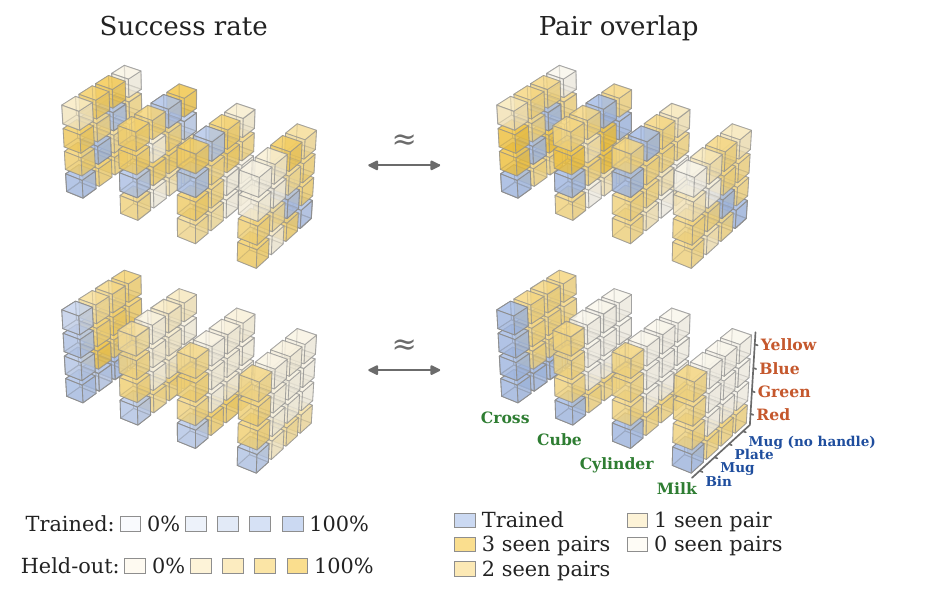}
        \caption{Representative PPP training sets ($B=10$; rows). \textbf{Left:} success rate, where darker means higher success. \textbf{Right:} number of seen instruction pairs for each held-out task. Held-out tasks with more seen pairs tend to achieve higher success.}
        \label{fig:pair_alignment_cubes}
    \end{subfigure}
    \caption{OOD success in PPP increases with the number of seen instruction pairs within each training budget $B$ (left), and in two example $B{=}10$ training sets the held-out success pattern closely matches the pair-overlap pattern (right).}
    \label{fig:pair_alignment}
    \vspace{-12pt}
\end{figure}

\begin{wrapfigure}{R}{0.55\linewidth}
    \vspace{-4pt}
    \centering
    \includegraphics[width=\linewidth]{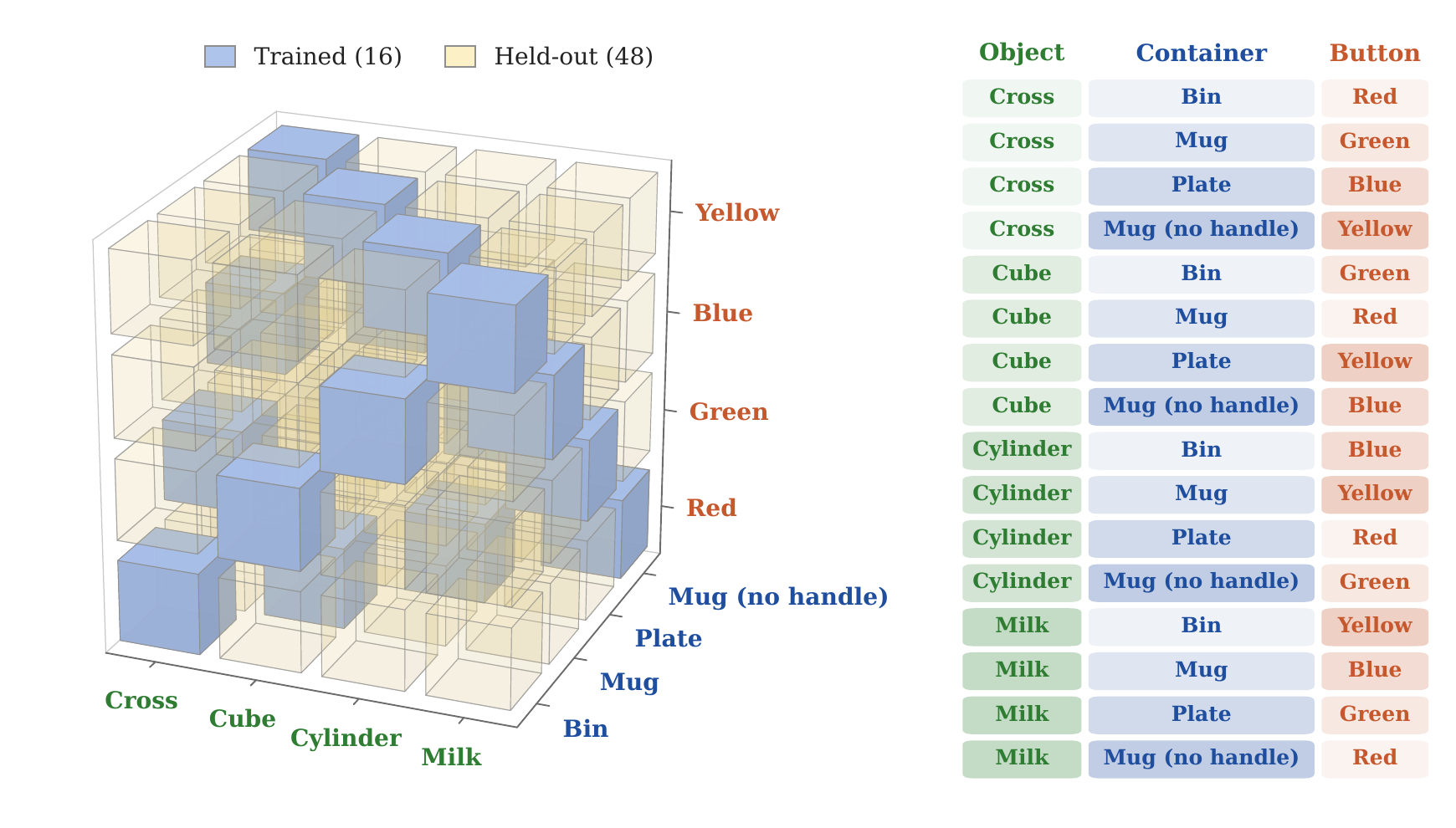}
    \vspace{-22pt}
    \caption{Orthogonal PPP training set. The 16 training tasks (blue; listed at right) cover every pairwise combination of instruction values exactly once.}
    \label{fig:orthogonal_set}
    \vspace{-16pt}
\end{wrapfigure}

These results in PPP task suggest that pairwise coverage can support strong OOD generalization. Accordingly, the orthogonal \(16\)-task design (Figure~\ref{fig:orthogonal_set}), inspired by design-of-experiments theory~\citep{kacker1991taguchi, hedayat2012orthogonal}, ensures that every pairwise combination of instruction factors appears in training. Experiments show that it preserves the relevant pairwise coverage using only one quarter of the full \(64\)-task space and achieves an OOD success rate of \(78.2\%\), close to the overall success rate of \(78.6\%\) achieved with full-task training. Results for all 55 PPP training-set configurations and an experiment with human demonstrations are provided in Appendix~\ref{app:q1}. It is worth noting that longer tasks may require coverage of higher-order interactions while still using far fewer tasks than full enumeration.

\begin{wrapfigure}{R}{0.55\linewidth}
    \vspace{-12pt}
    \centering
    \includegraphics[width=\linewidth]{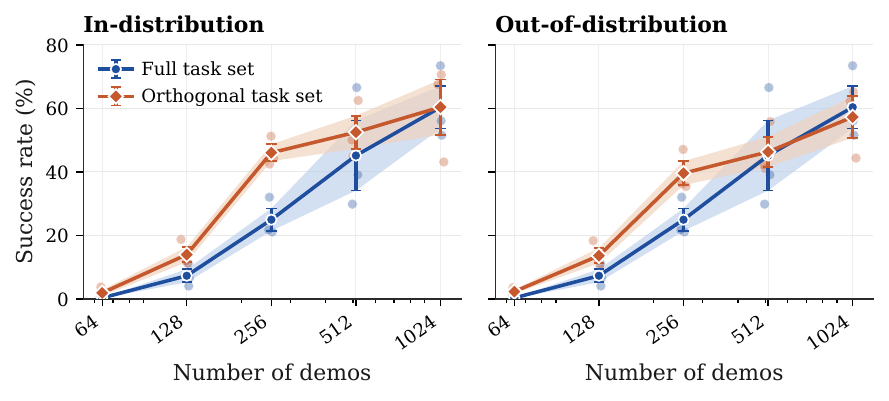}
    \vspace{-16pt}
    \caption{Orthogonal-set training outperforms full-set training when demonstrations are limited.}
    \label{fig:budget}
    \vspace{-12pt}
\end{wrapfigure}

\paragraph{Fixed-budget comparison.}
Under a fixed demonstration budget, exhaustive task coverage spreads demonstrations across many task tuples, whereas an orthogonal task set concentrates them on structurally informative combinations. As shown in Figure~\ref{fig:budget}, orthogonal-set training outperforms full-task training at low and medium budgets, suggesting that, in this regime, allocating more demonstrations to each selected task is more beneficial than expanding task coverage. The performance gap narrows as the total budget increases and more demonstrations become available per task. These results suggest that avoiding full enumeration can improve both task efficiency and sample efficiency.


\subsection{Q2: What Is Missing When Sparse Coverage Fails?}
\label{sec:exp_instruction_steering}

The preceding experiments show that sparse task sets can fail even when all subtask instruction values are observed. Proposition~\ref{prop:general_gap_bound} points to two remaining sources of failure: \textit{instruction-compositional shift} and \textit{context--action shift}. This raises an operational question: does the policy fail from lack of subtask skills, or from poor skill routing under unseen instruction combinations?


\begin{figure}[!htbp]
    \vspace{-12pt}
    \centering
    \includegraphics[width=0.95\linewidth]{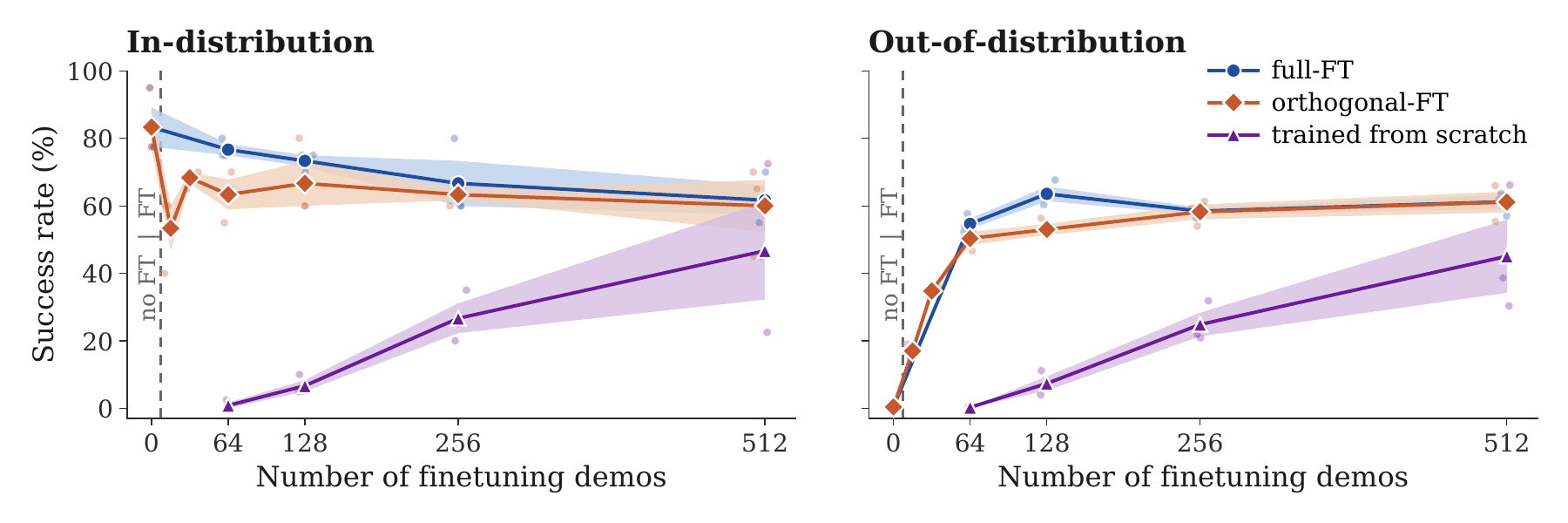}
    \vspace{-12pt}
    \caption{
    Finetuning from sparse diagonal pretraining. Small amounts of broad-coverage data substantially improve OOD success, suggesting that sparse pretraining learns reusable skills but not reliable instruction steering, which can be acquired with limited additional data.
    }
    \label{fig:ft}
    \vspace{-8pt}
\end{figure}

To separate these hypotheses, we use the Pick-Place-Press (PPP) task and pretrain on the diagonal set, i.e., the minimal task set $(o,c,b)\in\{(0,0,0),(1,1,1),(2,2,2),(3,3,3)\}$, which covers all instruction marginals, and then finetune with a few demonstrations from broader task sets. Figure~\ref{fig:ft} shows that diagonal pretraining achieves high ID success but nearly zero OOD success. However, with only \textit{one} demonstration per full-set task, finetuning raises OOD success from $0.4\%$ to $54.7\%$, whereas scratch training with the same total budget remains substantially weaker. ID success decreases after finetuning, suggesting a trade-off induced by distributional reweighting rather than a loss of subtask competence. Finetuning on orthogonal set shows a similar trend. Further analysis is in Appendix~\ref{app:q2}.

\begin{wrapfigure}{r}{0.5\linewidth}
    \vspace{-12pt}
    \centering
    \includegraphics[width=\linewidth]{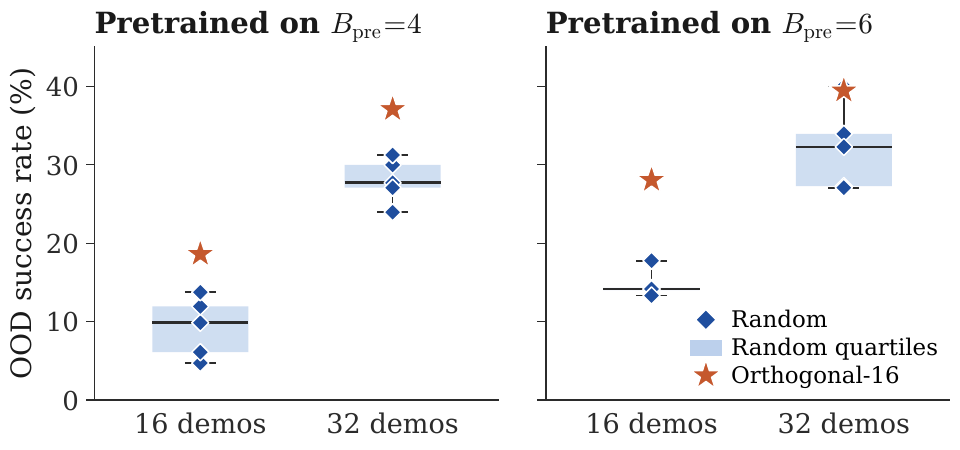}
    \vspace{-12pt}
    \caption{OOD success with 16 or 32 fine-tuning demonstrations. Stars: orthogonal set (16 tasks, one or two demonstrations per task); diamonds: five random sets (16 or 32 tasks, one demonstration per task); boxes: interquartile ranges.}
    \label{fig:ft_16b_32b}
    \vspace{-16pt}
\end{wrapfigure}

To assess the benefits of training task structure under severe data scarcity, we compare it with random selection using matched fine-tuning budgets and the same pretrained checkpoint within each comparison. Figure~\ref{fig:ft_16b_32b} shows that orthogonal selection achieves higher OOD success for both sparse pretraining sets (\(B=4,6\)), showing the value of structured coverage under scarce finetuning data.

These results indicate that sparse pretraining learns reusable subtask skills, but not the mapping from unseen instruction tuples to the appropriate behavior. That means sparse coverage pretraining leaves large instruction-compositional discrepancies; a small amount of broad-coverage finetuning can substantially reduce this gap; and well-structured task coverage can further improve performance when fine-tuning data are scarce.


\subsection{Q3: What Coverage Is Needed for Dependent Instructions?}
\label{sec:exp_dependent_two_stage}

Instruction dependence arises when the appropriate action depends on \(l_{-i}\) even after conditioning on the current context \(z\) and active instruction \(l_i\). Corollary~\ref{cor:modular_gap_bound} removes the explicit instruction-compositional shift term, but conditional context and action shift can remain because \(p(z,a\mid l_i)\) marginalizes over \(l_{-i}\). We therefore examine which instruction dependencies require training coverage and whether oracle stage information suffices when relevant instructions are omitted.

In 2S-PP, the policy receives \(l=(o_1,o_2,c_2)\) and must place the two objects into different containers. The first placement therefore depends on the later instruction \(c_2\). We first use four choices of \(o_1\), three remaining choices of \(o_2\), and three choices of \(c_2\), yielding \(4\times3\times3=36\) tasks. Starting from an orthogonal set of 12 training tasks covering all 12 dependent pairs \((o_1,c_2)\) (Cov12), we reduce pair coverage to 9, 6, and 4 (Cov9, Cov6, and Cov4), while keeping the number of training tasks fixed.

\begin{figure}[!t]
    \centering
    \includegraphics[width=0.99\linewidth]{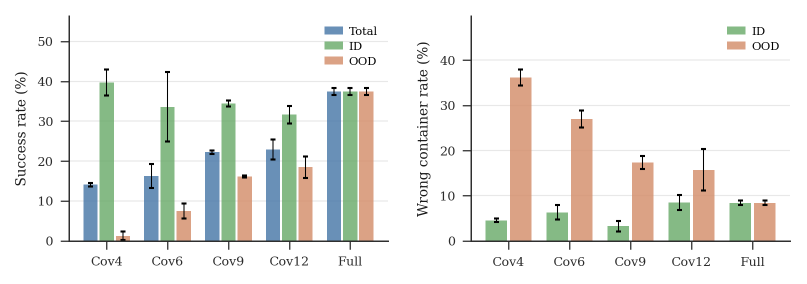}
    \caption{
    Performance and same-container violations in 2S-PP.
    Broader coverage of dependent pairs improves success and reduces relational errors, such as placing $o_1$ into $c_2$.
    }
    \label{fig:cont0}
    \vspace{-12pt}
\end{figure}

Figure~\ref{fig:cont0} shows that broader coverage of \((o_1,c_2)\) improves OOD success and reduces same-container violations. Conversely, fewer observed pairs yield higher ID success and fewer ID violations, which we attribute to denser supervision per observed pair. These results suggest that training coverage should reflect the dependencies governing action selection: joint combinations of semantically dependent instruction factors require broader coverage.
Further experimental details, coverage results, and failure analyses are provided in Appendix~\ref{app:q3}.

\begin{wrapfigure}{r}{0.6\textwidth}
    \vspace{-10pt}
    \centering
    \begin{subfigure}[t]{0.48\linewidth}
        \centering
        \includegraphics[width=\linewidth]{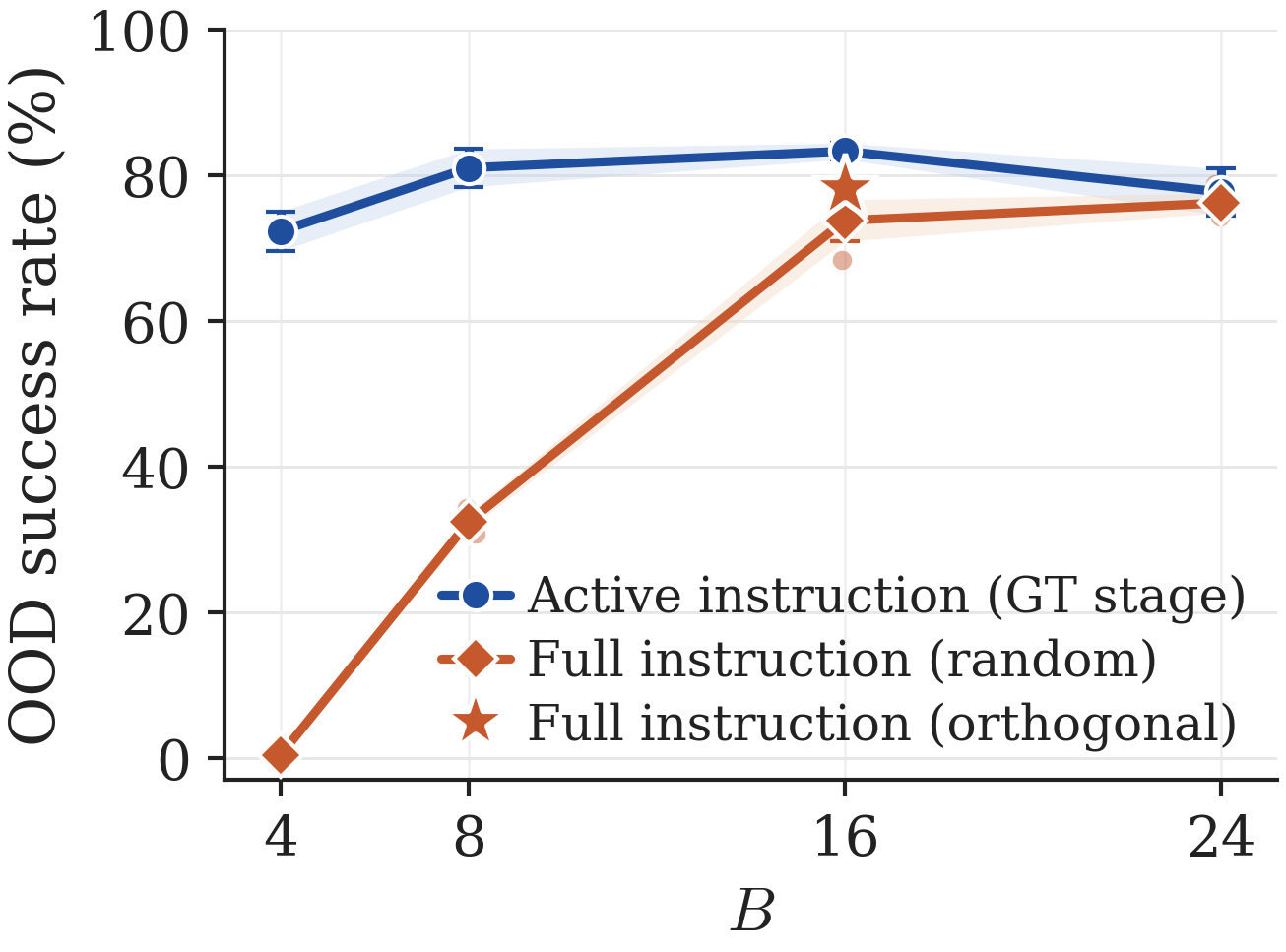}
        \caption{Task with independent instructions (PPP).}
        \label{fig:gt_indicator_ppp}
    \end{subfigure}\hfill
    \begin{subfigure}[t]{0.48\linewidth}
        \centering
        \includegraphics[width=\linewidth]{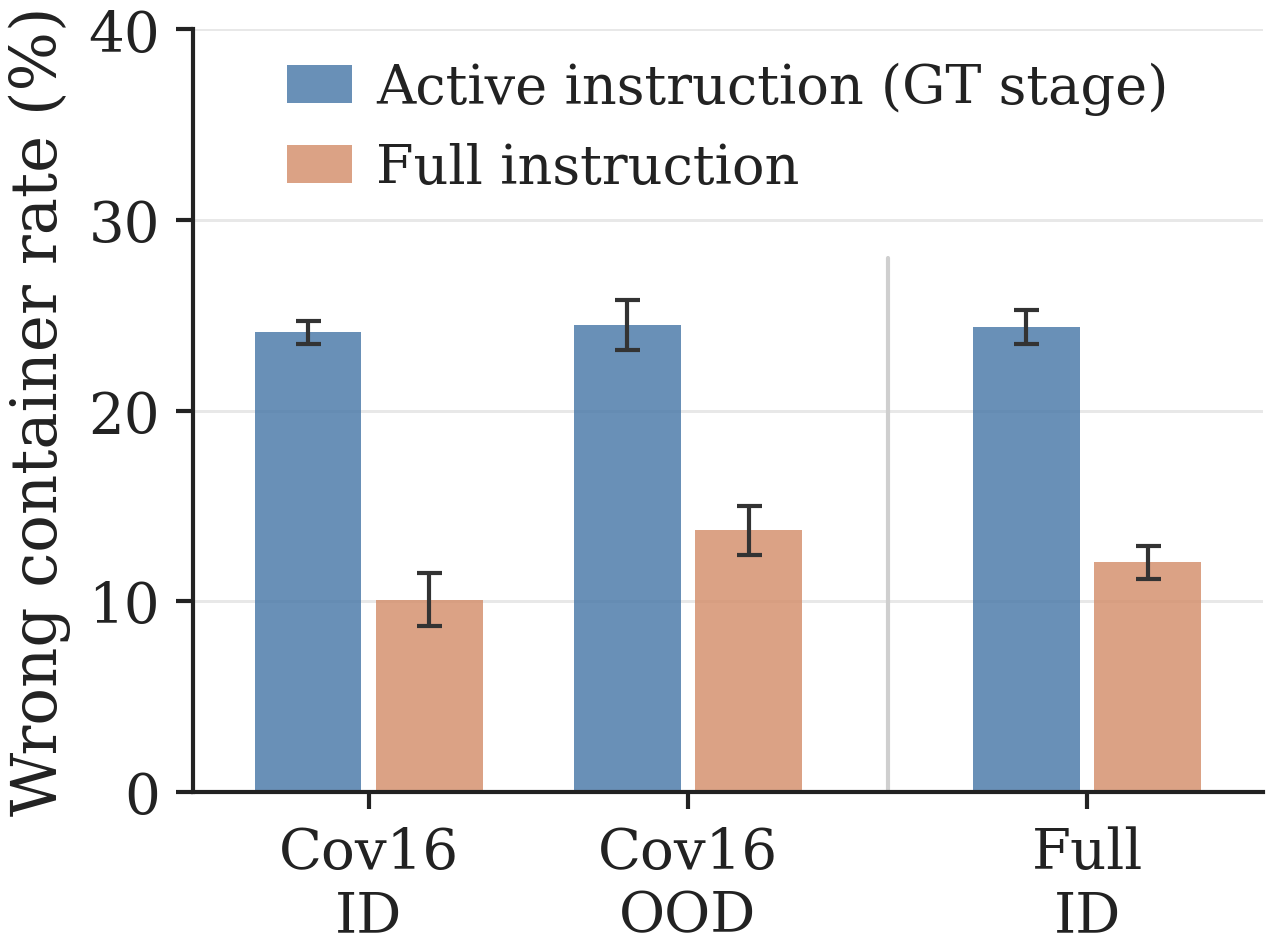}
        \caption{Task with dependent instructions (2S-PP).}
        \label{fig:gt_indicator_2spp}
    \end{subfigure}
    \caption{GT stage selects the active instruction; Full instruction uses all factors. Cov16 covers all dependent pairs; Full covers all 48 tasks.}
    \label{fig:gt_indicator}
    \vspace{-4pt}
\end{wrapfigure}

\paragraph{Comparison with stage-wise modular policies.}
We construct a ground-truth stage indicator \(\sigma(z)\) and compare a policy conditioned on the full instruction with the stage-wise modular policy described in Section~\ref{sec: stage wise modular policy}. Implementation details, including the stage indicator design, are provided in Appendix~\ref{app:oracle_stage}. As shown in Figure~\ref{fig:gt_indicator_ppp}, with an oracle selecting the active instruction at each stage, the stage-wise modular policy achieves strong OOD performance on PPP under sparse coverage, with a smaller compositional generalization gap consistent with Corollary~\ref{cor:modular_gap_bound}. In 2S-PP, however, choosing a valid container for the first placement requires \(c_2\), which is absent from the active instruction available to the modular policy. Thus, for fixed \(z\), \(i\), and \(l_i\), its actions cannot depend on \(c_2\), regardless of training coverage. The end-to-end full-instruction policy receives \(c_2\) and can learn this cross-stage constraint through optimized task coverage, specifically full coverage of \((o_1,c_2)\) in our setting. It achieves substantially lower OOD same-container violation rates under Cov16, and this advantage persists under Full (Figure~\ref{fig:gt_indicator_2spp}). These results motivate our focus on full-instruction policies, which retain the information needed to learn cross-stage dependencies and benefit from task coverage for compositional generalization.

%% file: sections/app.tex
\newpage
\appendix

\section{Theoretical Analysis}
\label{app: theoretical analysis}
\begin{proposition}[General upper bound on the compositional generalization gap]
\label{prop:general_gap_bound_appendix}
Assume the loss is measurable and uniformly bounded on
\(\operatorname{supp}(p)\cup \operatorname{supp}(q)\), i.e.,
\[
0 \le L_\theta(z,a,l) \le M
\qquad
\forall (z,a,l)\in \operatorname{supp}(p)\cup \operatorname{supp}(q).
\]
Then, for any fixed \(i \in \{1,\dots,n\}\),
\begin{equation}
\label{eq:general_gap_bound_appendix}
\begin{aligned}
|\Delta_q(\theta)|
\le M \Big(
&\|q(l_i)-p(l_i)\|_1 \\
&+ \mathbb{E}_{l_i \sim p(l_i)}
\big[\|q(l_{-i}\mid l_i)-p(l_{-i}\mid l_i)\|_1\big] \\
&+ \mathbb{E}_{l \sim p(l)}
\big[\|q(z,a\mid l)-p(z,a\mid l)\|_1\big]
\Big).
\end{aligned}
\end{equation}
\end{proposition}

\begin{proof}
Starting from
\[
\Delta_q(\theta)
=
\int_{\mathcal Z \times \mathcal A \times \mathcal L}
\bigl(q(z,a,l)-p(z,a,l)\bigr)\,L_\theta(z,a,l)\,dz\,da\,dl,
\]
and using the factorization \(l=(l_i,l_{-i})\), we have
\[
q(z,a,l)-p(z,a,l)
=
q(l_i)q(l_{-i}\mid l_i)q(z,a\mid l)
-
p(l_i)p(l_{-i}\mid l_i)p(z,a\mid l).
\]
Add and subtract the intermediate terms
\[
p(l_i)q(l_{-i}\mid l_i)q(z,a\mid l)
\quad\text{and}\quad
p(l_i)p(l_{-i}\mid l_i)q(z,a\mid l),
\]
to get
\begin{equation}
    \begin{aligned}
        q(z,a,l)-p(z,a,l) = & q(l_i)q(l_{-i}\mid l_i)q(z,a\mid l)
            -
            p(l_i)q(l_{-i}\mid l_i)q(z,a\mid l)\\
            &+ 
            p(l_i)q(l_{-i}\mid l_i)q(z,a\mid l)
            -
            p(l_i)p(l_{-i}\mid l_i)q(z,a\mid l)\\
            &+
            p(l_i)p(l_{-i}\mid l_i)q(z,a\mid l)
            -
            p(l_i)p(l_{-i}\mid l_i)p(z,a\mid l)
    \end{aligned}
\end{equation}

Then we have the exact decomposition
\begin{equation}
\label{eq:general_exact_decomp}
\Delta_q(\theta)=A_i+B_i+C_i,
\end{equation}
where
\[
A_i
=
\int
\bigl(q(l_i)-p(l_i)\bigr)\,
q(l_{-i}\mid l_i)\,
q(z,a\mid l)\,
L_\theta(z,a,l)\,dz\,da\,dl,
\]
\[
B_i
=
\int
p(l_i)\,
\bigl(q(l_{-i}\mid l_i)-p(l_{-i}\mid l_i)\bigr)\,
q(z,a\mid l)\,
L_\theta(z,a,l)\,dz\,da\,dl,
\]
\[
C_i
=
\int
p(l_i)\,
p(l_{-i}\mid l_i)\,
\bigl(q(z,a\mid l)-p(z,a\mid l)\bigr)\,
L_\theta(z,a,l)\,dz\,da\,dl.
\]

Using \(0\le L_\theta \le M\), each term is bounded as follows:
\[
|A_i|
\le
M
\int
|q(l_i)-p(l_i)|\,
q(l_{-i}\mid l_i)\,
q(z,a\mid l)\,dz\,da\,dl
=
M\|q(l_i)-p(l_i)\|_1,
\]
\[
|B_i|
\le
M
\int
p(l_i)\,
|q(l_{-i}\mid l_i)-p(l_{-i}\mid l_i)|\,
q(z,a\mid l)\,dz\,da\,dl
=
M\,
\mathbb{E}_{l_i\sim p(l_i)}
\big[
\|q(l_{-i}\mid l_i)-p(l_{-i}\mid l_i)\|_1
\big],
\]
and
\[
|C_i|
\le
M
\int
p(l)\,
|q(z,a\mid l)-p(z,a\mid l)|\,dz\,da\,dl
=
M\,
\mathbb{E}_{l\sim p(l)}
\big[
\|q(z,a\mid l)-p(z,a\mid l)\|_1
\big].
\]
Combining these three inequalities with \eqref{eq:general_exact_decomp} yields
\eqref{eq:general_gap_bound_appendix}.
\end{proof}

From Proposition~\ref{prop:general_gap_bound_appendix}, bounded training loss alone does not guarantee good compositional generalization. Instead, a small gap requires controlling three distinct discrepancies, each of which has a different interpretation and a different degree of controllability.

First, the marginal term \(\|q(l_i)-p(l_i)\|_1\) measures whether the active subtask \(l_i\) itself is covered by the training distribution. If test-time subtasks are absent from training, then no policy can be expected to generalize reliably, since the problem is one of missing support rather than model structure. This term is therefore fundamentally a data issue: it can only be reduced by expanding the training distribution to include the relevant subtasks.

Second, the compositional term
\[
\mathbb{E}_{l_i\sim p(l_i)}
\big[\|q(l_{-i}\mid l_i)-p(l_{-i}\mid l_i)\|_1\big]
\]
captures whether familiar subtasks are recombined with different surrounding instructions at test time. Unlike the first term, this shift arises even when each individual subtask has been observed during training. It reflects a failure to generalize across novel combinations of known factors, and thus directly characterizes the core challenge of compositional generalization.

Third, the conditional term
\[
\mathbb{E}_{l\sim p(l)}
\big[\|q(z,a\mid l)-p(z,a\mid l)\|_1\big]
\]
measures the residual shift in context--action behavior conditioned on the full instruction sequence. In practice, this term can grow when the learned policy becomes overly sensitive to irrelevant correlations among instruction components, so that unfamiliar recombinations of otherwise familiar subtasks induce different action behavior.

This decomposition suggests that the first term is unavoidable without broader training coverage, whereas the second and third terms depend more strongly on how the policy represents and uses instruction structure. In particular, if the policy can isolate the instruction component that is relevant at the current execution stage while suppressing unnecessary dependence on inactive subtasks, then its sensitivity to novel recombinations may be reduced. This motivates introducing an idealized stage-wise modular policy as a conceptual reference point.

\begin{corollary}[Modular bound]
\label{cor:modular_gap_bound_appendix}
Assume that for each stage \(i\),
\[
0 \le \ell_{i,\theta}(z,a,l_i) \le M_i
\qquad
\forall (z,a,l_i)\in \operatorname{supp}(p_i)\cup \operatorname{supp}(q_i).
\]
Then
\begin{equation}
\label{eq:modular_gap_bound_appendix}
\begin{aligned}
|\Delta_q^{\mathrm{m}}(\theta)|
\le
\sum_{i=1}^n M_i \Big(
&\|q(l_i)-p(l_i)\|_1 \\
&+
\mathbb{E}_{l_i\sim p}\|q(z,a\mid l_i)-p(z,a\mid l_i)\|_1\Big).
\end{aligned}
\end{equation}
\end{corollary}

\begin{proof}
Starting from
\[
\Delta_q^{\mathrm{m}}(\theta)
=
\int
\bigl(q(z,a,l)-p(z,a,l)\bigr)\,
L_\theta(z,a,l)\,dz\,da\,dl,
\]
and using
\[
L_\theta(z,a,l)
=
\sum_{i=1}^n \mathbf 1\{\sigma(z)=i\}\,\ell_{i,\theta}(z,a,l_i),
\]
we obtain
\[
\Delta_q^{\mathrm{m}}(\theta)
=
\sum_{i=1}^n
\int
\mathbf 1\{\sigma(z)=i\}\,
\ell_{i,\theta}(z,a,l_i)\,
\bigl(q(z,a,l)-p(z,a,l)\bigr)\,dz\,da\,dl.
\]
Since \(\ell_{i,\theta}(z,a,l_i)\) depends only on \((z,a,l_i)\), we can integrate out
\(l_{-i}\) and get
\[
\Delta_q^{\mathrm{m}}(\theta)
=
\sum_{i=1}^n
\int
\ell_{i,\theta}(z,a,l_i)\,
\bigl(q(z,a,l_i)-p(z,a,l_i)\bigr)\,dz\,da\,dl_i.
\]
Then we have
\[
q(z,a,l_i)-p(z,a,l_i)
=
\bigl(q(l_i)-p(l_i)\bigr)\,q(z,a\mid l_i)
+
p(l_i)\,\bigl(q(z,a\mid l_i)-p(z,a\mid l_i)\bigr).
\]
Therefore,
\[
\Delta_q^{\mathrm{m}}(\theta)
=
\sum_{i=1}^n (U_i+V_i),
\]
where
\[
U_i
=
\int
\ell_{i,\theta}(z,a,l_i)\,
\bigl(q(l_i)-p(l_i)\bigr)\,
q(z,a\mid l_i)\,dz\,da\,dl_i,
\]
and
\[
V_i
=
\int
\ell_{i,\theta}(z,a,l_i)\,
p(l_i)\,
\bigl(q(z,a\mid l_i)-p(z,a\mid l_i)\bigr)\,dz\,da\,dl_i.
\]
Using \(0\le \ell_{i,\theta}\le M_i\), we bound
\[
|U_i|
\le
M_i
\int
|q(l_i)-p(l_i)|
\left(\int q(z,a\mid l_i)\,dz\,da\right)\,dl_i
=
M_i\,\|q(l_i)-p(l_i)\|_1,
\]
and
\[
|V_i|
\le
M_i
\int
p(l_i)
\left(
\int
|q(z,a\mid l_i)-p(z,a\mid l_i)|\,dz\,da
\right)\,dl_i
=
M_i
\mathbb{E}_{l_i\sim p}
\|q(z,a\mid l_i)-p(z,a\mid l_i)\|_1.
\]
Summing over \(i\) yields \eqref{eq:modular_gap_bound_appendix}.
\end{proof}

\section{Experimental Details}
\label{app:experimental_details}

\subsection{Data Generation and Training Protocol}
\label{app:data_training}
We use a customized motion planner to generate 500, 300, and 100 demonstrations per task for PP, PPP, and 2S-PP, yielding 8,000, 19,200, and 14,400 demonstrations in total, respectively. These counts describe the generated demonstration pools. Each experiment selects demonstrations from its training-task subset according to the specified budget.

For comparisons across training-set sizes \(B\), we use the same total number of demonstrations, limited by the data available in the set with the fewest tasks, and allocate them uniformly across the \(B\) training tasks. Thus, full-task training has the smallest per-task allocation. Random seeds control demonstration sampling and network initialization. PPP policies are trained for 100k steps with a batch size of 64; the dependent 2S-PP policies are trained for 150k steps. Experiment-specific finetuning settings are described in Appendix~\ref{app:q2}.

\subsection{Policy Architecture}
\label{app:network}
All experiments use the same transformer-based flow-matching policy with identical observation and action spaces. At each time step, the context \(z\) contains visual observations from third-person and eye-in-hand cameras and robot state, including the 6D gripper pose. The action \(a\) consists of end-effector position and orientation deltas, together with gripper actuation.

We use a frozen DINOv2 model as the image encoder. Each instruction is represented as a token, while multi-instruction inputs are represented by multiple instruction tokens that serve as queries. The image tokens serve as keys and values, following a design similar to Q-Former~\cite{li2023blip}. The resulting output tokens are then injected into the flow-matching transformer through cross-attention. Further details are shown in Figure~\ref{fig:policy network}.

\begin{figure}[t]
    \centering
    \includegraphics[width=0.9\linewidth]{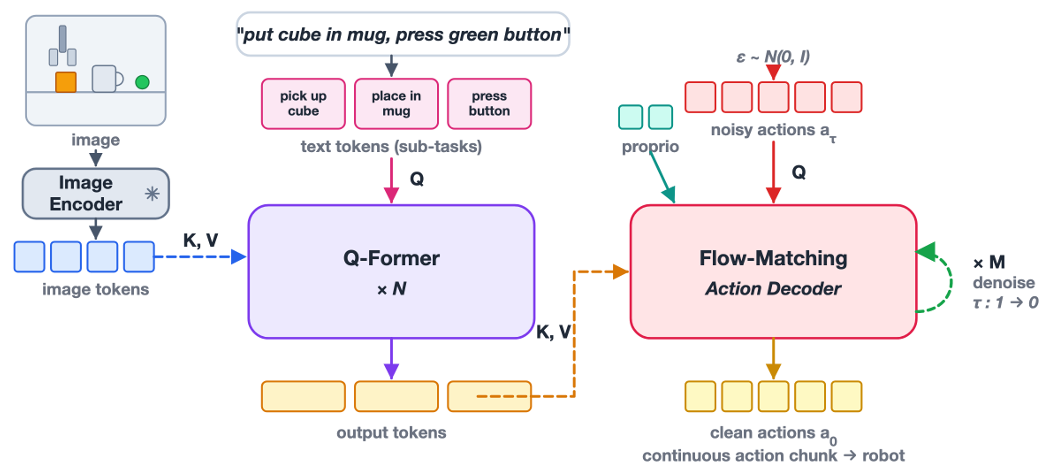}
    \caption{
    Policy network structure.
    }
    \label{fig:policy network}
    \vspace{-12pt}
\end{figure}

\subsection{Evaluation Protocol}
\label{app:evaluation}
Unless otherwise noted, we evaluate each trained policy for 30 trials per task using three training seeds. We compute each metric separately for each seed and report its mean across seeds. Total SR, ID SR, and OOD SR are the fractions of successful rollouts on \(E_{\mathrm{total}}\), \(E_{\mathrm{train}}\), and \(E_{\mathrm{ood}}\), respectively. For full-task baselines, the overall success rate is reported in both ID and OOD success plots; it is not an OOD evaluation. Appendix~\ref{app:q3} specifies the seed counts and evaluation settings for the 48-task 2S-PP experiments and oracle stage comparisons. For 2S-PP, the same-container violation rate is the fraction of rollouts that place \(o_1\) into \(c_2\), among rollouts that complete the first placement. Rollouts that fail before completing the first placement are excluded from this denominator.

\begin{wraptable}{r}{0.4\textwidth}
    \vspace{-12pt}
    \centering
    \caption{Sensitivity to evaluation trials per task on 2S-PP.}
    \label{tab:eval_trials}
    \resizebox{\linewidth}{!}{%
        \begin{tabular}{rcc}
            \toprule
            Trials per task & ID SR (\%) & OOD SR (\%) \\
            \midrule
            5  & 56.25 & 15.00 \\
            10 & 63.75 & 15.70 \\
            15 & 62.92 & 16.30 \\
            30 & 61.88 & 15.89 \\
            \bottomrule
        \end{tabular}%
    }
    \vspace{-15pt}
\end{wraptable}

\paragraph{Sensitivity to evaluation trial count.}
We vary the number of evaluation trials per task on 2S-PP from 5 to 30, as shown in Table~\ref{tab:eval_trials}. Increasing from 15 to 30 trials changes ID and OOD success rates by only 1.04 and 0.41 percentage points, respectively, supporting our choice of 30 trials per task for evaluation.

\section{Q1: Is Full Tuple Coverage Necessary?}
\label{app:q1}
\paragraph{Pick-Place-Press.}We trained the policy for 100k steps with a batch size of 64 on the Pick-Place-Press (PPP) task, conducting experiments across 55 training-set configurations in total. Figure~\ref{fig:whole b log} shows a similar trend to Figure~\ref{fig:ppp_budget}. Figure~\ref{fig:pool seen pair} shows that OOD tasks with greater seen-pair coverage achieve higher success rates. Additional results, in both image and HTML formats, are provided in the supplementary materials and on the \href{https://yixiaowang7.github.io/Diagnosing_Compositional_Generalization_Robot_Page/}{project website}.

\paragraph{Human demonstrations.}
We also evaluate a policy trained on human demonstrations in Isaac Lab~\cite{mittal2025isaaclab}, using a task with three objects, three containers, and three target stations, yielding \(3^3=27\) task combinations. The robot arm uses its gripper to place the selected object into the selected container and then move the container to the target station. We collect 30 demonstrations per task using a SpaceMouse and train on an orthogonal subset of \(B=9\) tasks. We evaluate three training seeds with 30 trials per task per seed, measuring ID success on the nine training tasks and OOD success on the 18 unseen combinations. Averaged across seeds, the policy achieves \(11.67\%\) ID and \(12.96\%\) OOD success. These comparable rates provide preliminary evidence of generalization under orthogonal coverage.

\begin{figure}[t]
    \centering
    \begin{subfigure}[t]{0.45\linewidth}
        \centering
        \includegraphics[width=\linewidth]{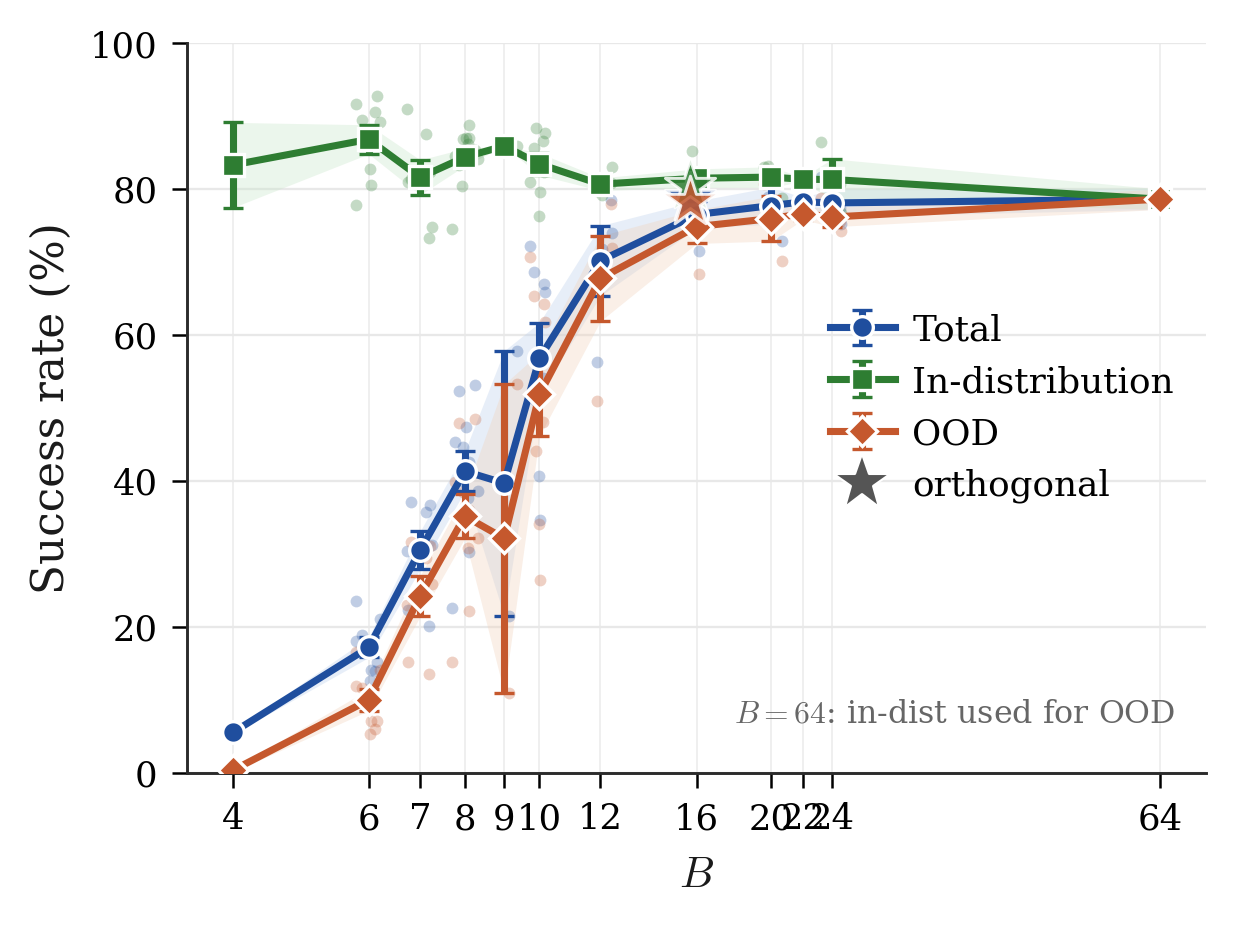}
        \caption{Success rate versus budget $B$ (full experiments). The orthogonal set achieves the same performance as the full task set using only one-fourth of the tasks.}
        \label{fig:whole b log}
    \end{subfigure}
    \hfill
    \begin{subfigure}[t]{0.45\linewidth}
        \centering
        \includegraphics[width=\linewidth]{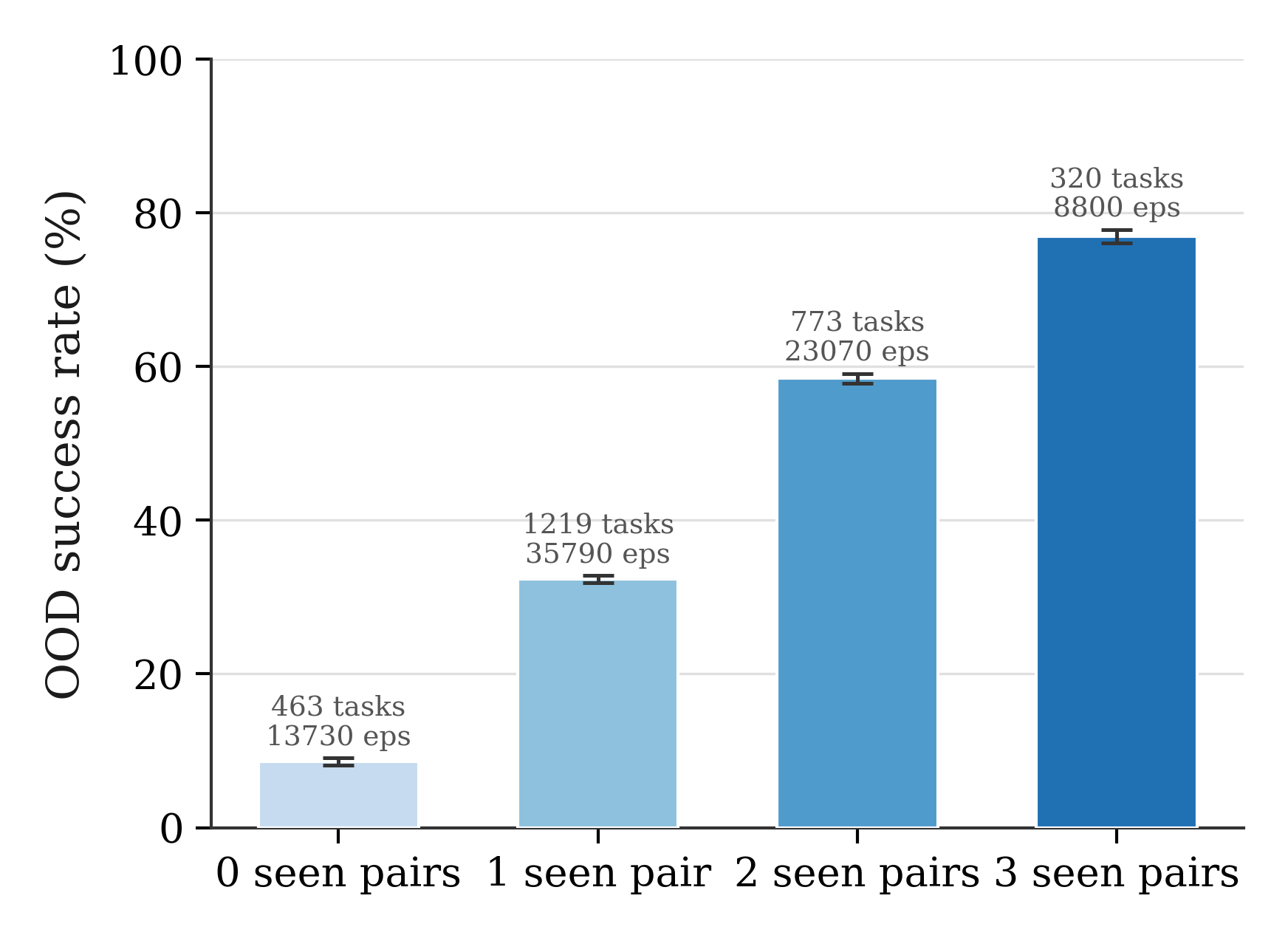}
        \caption{OOD tasks with more seen pairs achieve higher success.}
        \label{fig:pool seen pair}
    \end{subfigure}
    \caption{Additional experiments on the PPP task.}
    \label{fig: whole sets}
    \vspace{-12pt}
\end{figure}

\begin{figure}[t]
    \centering
    \includegraphics[width=0.95\linewidth]{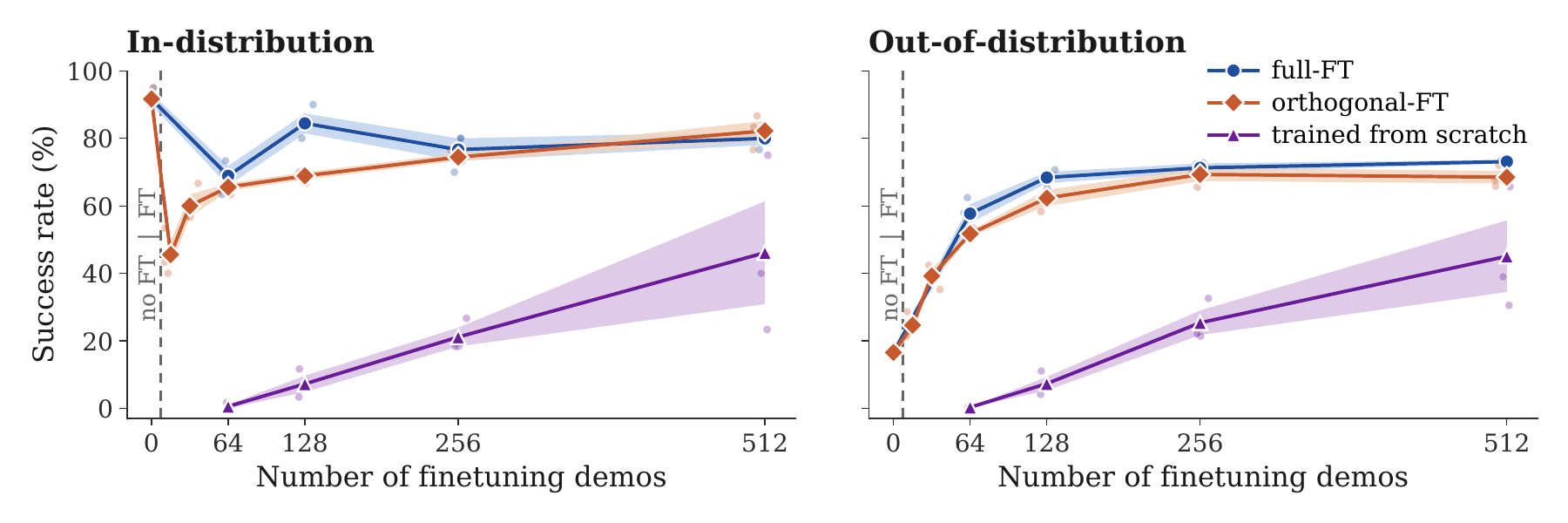}
    \caption{
    Finetuning from sparse pretraining ($\mathtt{diag\_03}, B=6$). Small amounts of broad-coverage data substantially improve OOD success, suggesting that sparse pretraining learns reusable skills but not reliable instruction steering, which can be acquired with limited additional data.
    }
    \label{fig:ft diag03}
    \vspace{-12pt}
\end{figure}

\newpage
\section{Q2: What Is Missing When Sparse Coverage Fails?}
\label{app:q2}

\begin{wrapfigure}{r}{0.45\linewidth}
    \vspace{-12pt}
    \centering
    \includegraphics[width=0.95\linewidth]{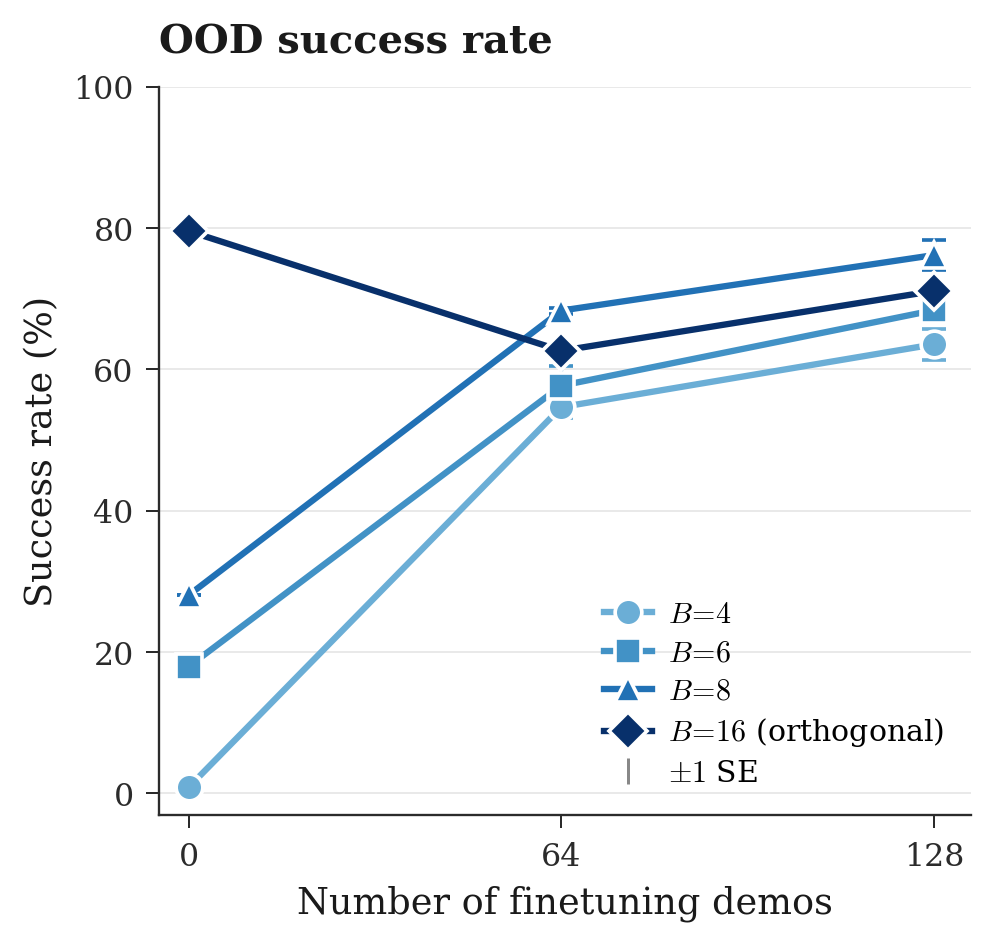}
    \caption{
    Finetuning from different pretrained models ($B=4,6,8,16$). \textbf{Only one demonstration} per task can significantly improve OOD success. However, finetuning can also suffer from overfitting when the pretrained model is already strong.
    }
    \label{fig:ft different B}
    \vspace{-12pt}
\end{wrapfigure}
We pretrained the model for 100k steps and then finetuned it for only 5k steps, using a batch size of 64 for both stages. In addition to pretraining on the minimal diagonal sets shown in Figure~\ref{fig:ft}, we also conduct experiments on another sparse-coverage set with $B=6$, as shown in Figure~\ref{fig:ft diag03}. The results lead to a similar conclusion: finetuning with only one demonstration per task significantly improves OOD success. However, the $B=6$ setting achieves higher OOD performance, likely because the pretrained checkpoint already exhibits stronger OOD generalization. We further conduct finetuning experiments using checkpoints pretrained with larger budgets, including $B=8$ and $B=16$ (orthogonal). Figure~\ref{fig:ft different B} shows that better pretrained checkpoints generally lead to better OOD performance after finetuning. However, when the pretrained checkpoint is already strong, as in the orthogonal setting, finetuning can degrade OOD performance. This may be because finetuning uses only 64 demonstrations, one per task, which can lead to overfitting to the finetuning dataset. A similar degradation is also observed on ID tasks.

\section{Q3: What Coverage Is Needed for Dependent Instructions?}
\label{app:q3}
\begin{figure}[t]
    \centering
    \includegraphics[width=0.99\linewidth]{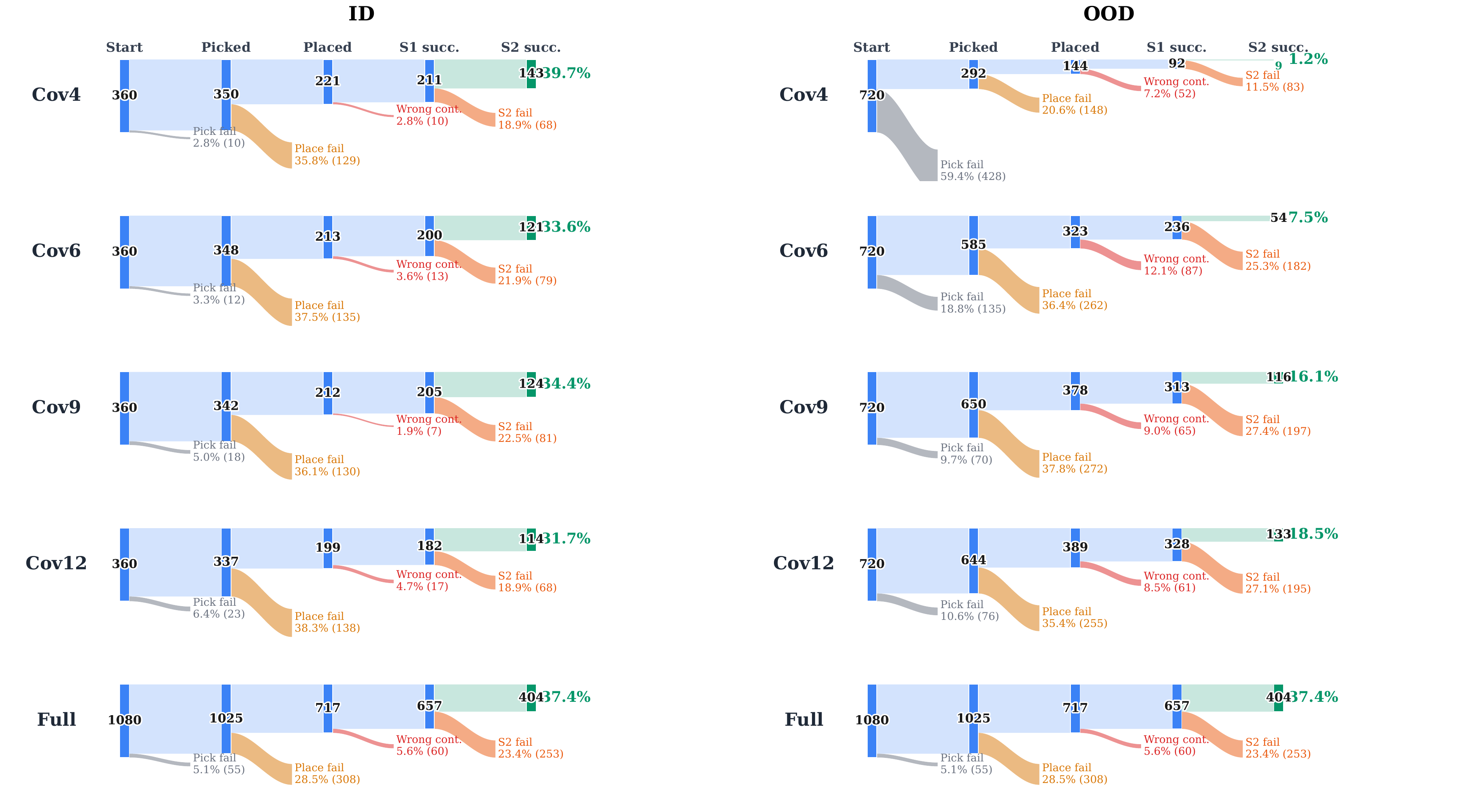}
    \caption{
    Failure case analysis on 2S-PP. Numbers indicate rollout counts.
    }
    \label{fig:sankey 36}
    \vspace{-12pt}
\end{figure}
To account for the increased difficulty of this task, we trained the policy for 150k steps.
\paragraph{36-task set.}The Sankey plot in Figure~\ref{fig:sankey 36} shows that ID performance is similar across settings. However, for OOD tasks, performance degrades significantly when there are fewer distinct $(o_1,c_2)$ pairs. This degradation appears not only as placing object $o_1$ into an incorrect container, but also as a substantially lower grasp success rate. This is likely because four distinct pairs provide insufficient coverage, causing $o_1$ and $c_2$ to become spuriously correlated. This effect is quickly mitigated when the coverage is increased to six pairs, as shown in Cov6, although the wrong-container rate remains high. These results suggest that instruction dependencies require broader coverage to avoid spurious correlations.

\paragraph{48-task set.} We extend the task set to \(4\times3\times4=48\) tuples by allowing four values of \(c_2\). Full instruction policies use five training seeds and 10 evaluation trials per task per seed. Starting from an orthogonal set of 16 training tasks covering all 16 dependent pairs \((o_1,c_2)\) (Cov16), we reduce pair coverage to 12, 9, and 6 (Cov12, Cov9, and Cov6), while keeping 16 training tasks. Figure~\ref{fig:cont0 48} shows that broader coverage improves OOD success and reduces same container violations. Fewer observed pairs yield higher ID success and fewer ID violations, consistent with denser supervision per pair. The failure patterns in Figure~\ref{fig:sankey 48} resemble those in the 36-task setting.

\paragraph{Ground-truth stage indicator.}\label{app:oracle_stage}
We construct the ground-truth stage indicator \(\sigma(z)\) used in our comparison with stage-wise modular policies to select the active instruction at each stage. A deterministic tracker implements this indicator using robot state and object and container poses, with identical rules for demonstration labeling and online evaluation. PPP has three stages: grasp the object, place it in the target container, and press the target button. The corresponding active instructions are \(o\), \(c\), and \(b\). The tracker switches from \(o\) to \(c\) upon detecting a grasp, and from \(c\) to \(b\) upon detecting release inside the target container. 2S-PP also has three indicator stages: grasp and place \(o_1\) in a container other than \(c_2\), grasp \(o_2\), and place \(o_2\) in \(c_2\). The corresponding active instructions are \(o_1\), \(o_2\), and \(c_2\). The tracker switches to \(o_2\) after a valid first placement and to \(c_2\) after detecting a grasp of \(o_2\). Although the tracker uses \(c_2\) to validate the first placement, the active instruction supplied to the policy at this stage contains only \(o_1\). This distinction underlies the information limitation discussed in Section~\ref{sec: stage wise modular policy} and Section~\ref{sec:exp_dependent_two_stage}. Grasp detection uses end-effector proximity and finger separation; placement requires geometric container membership without a detected grasp. Stage changes require two consecutive matching observations and may revert if a grasp or placement is lost. The tracking poses are not additional policy inputs, and episode success is evaluated separately.

\begin{figure}[t]
    \centering
    \includegraphics[width=0.99\linewidth]{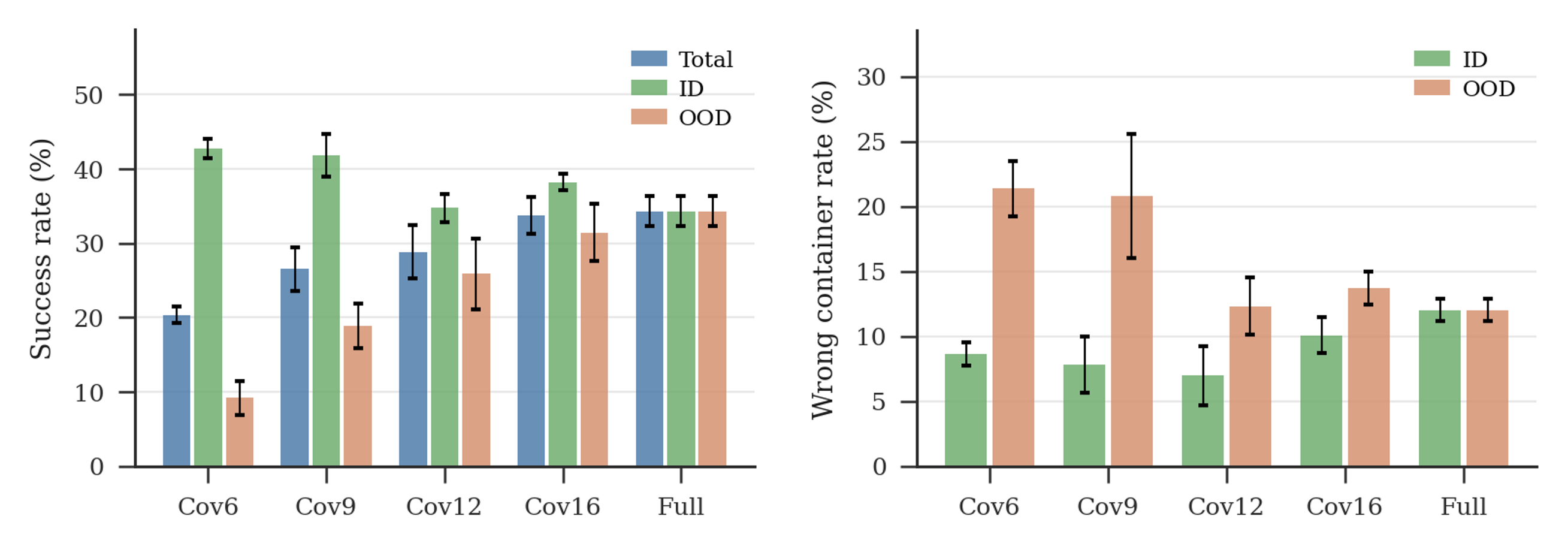}
    \caption{
    Performance and same-container violations in 2S-PP (48-task set).
    Broader dependent-pair coverage improves success and reduces relational errors (e.g., placing object $o_1$ into container $c_2$).
    }
    \label{fig:cont0 48}
    \vspace{-12pt}
\end{figure}

\begin{figure}[t]
    \centering
    \includegraphics[width=0.99\linewidth]{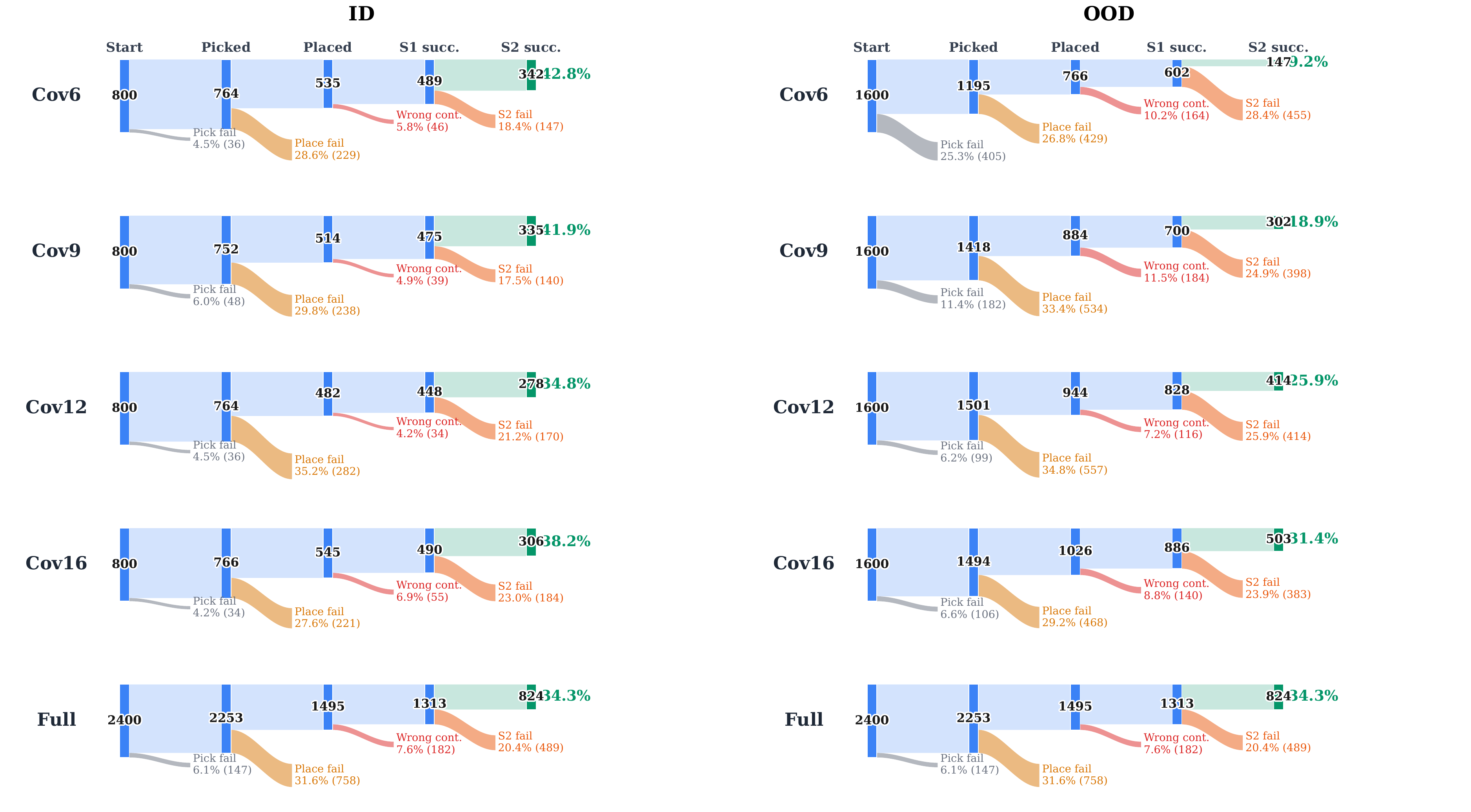}
    \caption{
    Failure case analysis on 2S-PP (48-task set). Numbers indicate rollout counts.
    }
    \label{fig:sankey 48}
    \vspace{-12pt}
\end{figure}

\paragraph{Factors limiting task success in 2S-PP.} Omitting \(c_1\) leaves multiple valid containers for the first placement under the constraint \(c_1\neq c_2\), making the demonstration distribution multimodal. In an auxiliary PP experiment across five training seeds, omitting the container instruction reduces success from \(71.05\%\pm5.00\%\) to \(49.95\%\pm1.14\%\). This is consistent with the frequent failures during placement of the first object in Figures~\ref{fig:sankey 36} and~\ref{fig:sankey 48}. Errors can also accumulate across the four execution stages of grasping and placing each object. We leave improved learning from multimodal demonstrations in these tasks to future work.